\documentclass{article}

\usepackage[preprint]{neurips_2020}
\usepackage{amsfonts}
\usepackage{times}
\usepackage{mathtools, enumitem}
\usepackage{amsthm} 
\usepackage{latexsym}
\usepackage{relsize}

\usepackage{mathrsfs, float}
\usepackage{textcomp}

\usepackage{hyperref}
\usepackage[capitalise]{cleveref}
\usepackage{xcolor}
\usepackage{dsfont}

\usepackage{algorithm,algorithmic}
\usepackage{natbib}
 \setcitestyle{numbers}

\usepackage{thm-restate}
\newtheorem{theorem}{Theorem}
\newtheorem{lemma}[theorem]{Lemma}

\newtheorem{corollary}[theorem]{Corollary}
\newtheorem{assumption}{Assumption}
\newtheorem{definition}{Definition}
\allowdisplaybreaks

\newcommand{\M}{\mathcal{M}}
\DeclareMathOperator{\Tr}{Tr}

\def\mL{{\mathcal L}}

\def\H{{\mathcal H}}

\newcommand{\D}{\mathcal{D}}

\newcommand{\A}{\mathcal{A}}

\newcommand{\sign}{\text{ } \mathrm{sign}}

\def\regret{\mbox{{Regret}}}

\newcommand{\ignore}[1]{}

\newcommand{\email}[1]{\texttt{#1}}

\DeclareMathAlphabet{\mathbfsf}{\encodingdefault}{\sfdefault}{bx}{n}

\let\Pr\relax
\DeclareMathOperator{\Pr}{\mathbb{P}}

\newcommand{\E}{\mathbb{E}}

\newcommand{\poly}{\mathrm{poly}}

\newcommand{\reals}{\mathbb{R}}
\newcommand{\eps}{\varepsilon}

\renewcommand{\leq}{~\le~}

\let\oldtfrac\tfrac
\renewcommand{\tfrac}[2]{\smash{\oldtfrac{#1}{#2}}}

\let\nablaold\nabla
\renewcommand{\nabla}{\nablaold\mkern-2.5mu}

\title{Black-Box Control for Linear Dynamical Systems}

\author{
   Xinyi Chen$^{1\,2}$\hspace{4em}  Elad Hazan$^{1\,2}$\\
  $^1$ Google AI Princeton \\ $^2$ Department of Computer Science, Princeton University \\
  \texttt{\{xinyic,ehazan\}@cs.princeton.edu}
  }
\usepackage{times}
\begin{document}

\maketitle

\begin{abstract}
We consider the problem of controlling an unknown linear time-invariant dynamical system from a single chain of black-box interactions, with no access to “resets” or offline simulation. Under the assumption that the system is  controllable, we give the first {\it efficient} algorithm that is capable of attaining sublinear regret in a single trajectory under the setting of online nonstochastic control. This resolves an open problem since the work of \cite{csaba} on the stochastic LQR problem, and in a more challenging setting that allows for adversarial perturbations and adversarially chosen and changing convex loss functions. 

We give finite-time regret bounds for our algorithm on the order of $2^{\tilde{O}(\mL)} + \tilde{O}(\poly(\mL) T^{2/3})$ for general nonstochastic control, and $2^{\tilde{O}(\mL)} + \tilde{O}(\poly(\mL) \sqrt{T})$ for black-box LQR, where $\mL$ is the system size  which is an upper bound on the dimension. The crucial step is a new system identification method that is robust to adversarial noise, but incurs exponential cost.

To complete the picture, we investigate the complexity of the online black-box control problem, and give a matching lower bound of $2^{\Omega(\mL)}$ on the regret, showing that the additional exponential cost is inevitable. This lower bound holds even in the noiseless setting, and applies to any, randomized or deterministic, black-box control method.  
\end{abstract}


\section{Introduction}

The ultimate goal in the field of adaptive control and reinforcement learning is to produce a truly independent learning agent. Such an agent can start in an unknown environment and follow one
continuous and uninterrupted chain of experiences, until it performs as well as the optimal policy. 

In this paper we consider this goal for the fundamental problem of controlling an unknown, linear time-invariant (LTI) dynamical system. This problem has received significant attention in the recent ML literature. However, nearly all existing methods assume some knowledge about the environment, usually in the form of a stabilizing controller. \footnote{Roughly speaking, a stabilizing controller is a policy that ensures the system will not explode, i.e. that the states stay bounded, under bounded perturbations. We formally define this concept in later sections. }  
The only exception is the seminal work of \cite{csaba}, which gives near-optimal regret bounds for certain variants of this problem, albeit using an exponential-time algorithm. It has been an open question in the literature to find an efficient algorithm. Our main contributions are resolving this  question, and matching our regret bounds with a lower bound that is tight for nonstochastic online LQR. 

We henceforth describe a control algorithm that {\bf only has black-box access} to an LTI system, meaning it has no knowledge of a stabilizing controller. The algorithm is guaranteed to attain sublinear regret, converging on average to the performance of the best controller in hindsight among a set of reference policies. Furthermore, its guarantees apply to the setting of nonstochastic control, in which both the perturbations and cost functions can be adversarially chosen.   The question of controlling unknown systems under adversarial noise was posed in \citep{tu2019sample}; our results quantify the difficulty of this task and provide a polynomial time solution. As far as we know, these results are the first finite-time sublinear regret bounds known for black-box, single-trajectory control in the nonstochastic model. Table 1 provides a summary of results.


Our regret bounds are accompanied by a novel lower bound on the cost of black-box control.  We show that this cost is inherently exponential in the natural parameters of the problem for {\it any}, deterministic or randomized, control method. As far as we know, this is the first finite-time lower bound for the online control problem that is exponential in the {\it system dimension}.


\subsection{Statement of results}\label{sec:results}

\begin{table}
\begin{tabular}{|c |c| c| c| c | } 
\hline
 Algorithm  & Regret Bound & Efficient & Disturbances & Cost Functions \\
 \hline
\citep{csaba} & $2^\mL \sqrt{T} $ & No & Stochastic & Quadratic \\ 
 \hline
Ours & $2^{\mL\log \mL} + \sqrt{T}$ & Yes & Adversarial & Strongly convex \\ 
 \hline
Ours & $2^{\mL\log \mL} + T^{2/3}$ & Yes & Adversarial & General convex \\
\hline
\end{tabular}
\caption{Summary of settings and results} \label{setting_table}
\end{table}

Consider a given LTI dynamical system with black-box access. The only interaction of the controller with the system is by sequentially observing states $x_t$ and applying controls $u_t$. The evolution of the state is according to the dynamics equation 
$$
x_{t+1} = Ax_t + Bu_t + w_t,
$$
where $x_t\in \mathbb{R}^{d_x}$, $u_t\in\mathbb{R}^{d_u}$. The system dynamics $A, B$ are unknown to the controller, and the disturbance $w_t$ can be adversarially chosen at the start of each time step. An adversarially chosen convex cost function $c_t(x, u)$ is revealed after the controller's action, and the controller suffers $c_t(x_t, u_t)$. 
In this model, a controller $\A$ is simply a mapping from all previous states and costs to a control. 
The total cost of executing a controller $\A$, whose  sequence of controls is denoted as $u_t^\A$, is defined as 
$$
J_T(\A) = \sum_{t=1}^T c_t(x_t^\A, u_t^\A).
$$
For a randomized control algorithm, we consider the expected cost. Under the special case of quadratic cost functions, if the perturbations are i.i.d. stochastic, it is referred to as the {\bf online LQR} setting; if the perturbations are adversarial, it is referred to as the {\bf nonstochastic online LQR} setting.

In the nonstochastic setting the optimal controller cannot be determined a priori, and depends on the disturbance realization. For this reason, we consider a comparative performance metric. 
The goal of the learning algorithm is to choose a sequence of controls $\{u_t\}_{t=1}^T$ such that the total cost over $T$ iterations is competitive with that of the best controller in a reference policy class $\Pi$. Thus, the learner {\bf with only black-box access, and in a single trajectory}, seeks to minimize regret defined as
$$
\regret_T(\A) = J_T(\A) - \min_{\pi \in \Pi} J_T(\pi).
$$
As a comparator class, we consider the set of all Disturbance-action Controllers (Definition \ref{def:dac}), whose control is a linear function of past disturbances. This class is known to contain state of the art Linear Dynamical Controllers (Definition \ref{def:ldc}), and the $H_2$ and $H_\infty$  optimal controllers.

Let $\mL$ denote the upper bound on the system's natural parameters, and $\kappa^*$ be the controllability parameter of the stabilized system (Section \ref{sec:notations}). Let $\tilde{\kappa}$ denote an upper bound on the stability parameters of the recovered controller (Section \ref{sec:nonstochastic}). The following statements summarize our main results in Theorem \ref{thm:main} and Theorem \ref{thm:lowerbound_rand}:  
\begin{enumerate}[leftmargin=0.5cm]
    \item We give an efficient algorithm, whose regret in a single trajectory is with high probability at most $$ \regret_T(\A) \le 2^{O(\mL \log \mL)} + \tilde{O} ( \poly(\mL, \kappa^*)  T^{2/3}) . $$
    
    \item For the nonstochastic online LQR problem, we give an efficient algorithm, whose regret in a single trajectory is with high probability at most $$ \regret_T(\A) \le 2^{O(\mL \log \mL)} + \tilde{O} ( \poly(\mL, \tilde{\kappa})  \sqrt{T}) . $$ 
    
    \item We show that {\bf any  control algorithm} (randomized or deterministic) must suffer a worst-case exponential start-up cost in the regret, due to limited information. Formally, we show that for every controller $\A$, there exists an LTI dynamical system where (with high probability if the algorithm is randomized)
    $$ \regret_T(\A) \ge 2^{\Omega(\mL)}. $$
    
    Further, this lower bound holds even without disturbances, and when the system-input matrix $B$ is full rank. From existing results by \cite{tomerlowerbound}, in general the online LQR problem has regret lower bound $\Omega(\sqrt{T})$. Therefore any algorithm must incur regret at least $2^{\Omega(\mL)} + \sqrt{T}$.
\end{enumerate}

To the best of our knowledge, these are the first finite-time regret bounds for control in a single trajectory with black-box access to the system under the nonstochastic setting, and the first polynomial time and optimal regret algorithm for nonstochastic black-box online LQR.
Our result quantifies the price of information for controlling unknown LTI systems to be exponential in the natural parameters of the problem, and shows that this is inevitable \footnote{See conclusions section for a discussion of intriguing open problems that remain.} . 

The main challenge of designing an efficient algorithm is obtaining a stabilizing controller from black-box interactions in the presence of adversarial noise and changing convex costs.  Our method consists of three phases. In the first phase, we identify the dynamics matrices up to some accuracy in the spectral norm by injecting large controls into the system. Previous works on system identification under adversarial noise either require stable dynamics, or the knowledge of a strongly stable controller. However, our approach is not limited by these requirements. 

In the second phase, we use an SDP relaxation for the LQR proposed by \cite{cohen2018online} to solve for a strongly stable controller given the system estimates. After we identify a strongly stable controller, we use the techniques of \cite{hazan2019nonstochastic} for regret minimization in the third phase for general convex costs, and those of \cite{simchowitz2020making} for the nonstochastic online LQR problem.


For the lower bound, our approach is inspired by lower bounds for gradient-based methods from the optimization literature \citep{gradientlowerbound}. We give two separate lower bounds: one for deterministic algorithms, and one for any algorithm. The deterministic lower bound has the advantage that it holds even for degenerate systems, and has better constants. Given a controller, we show system constructions that forces the states, and thus costs, to grow exponentially before enough information about the system is revealed. 

\subsection{Related work}

The focus of our work is adaptive control, in which the controller does not have a priori knowledge of the underlying dynamics, and has to learn them as well as control the system according to given convex costs. 
This task, recently put forth in the machine learning literature, differs substantially from the classical literature on control theory that we survey below in the following aspects:
\begin{enumerate}[leftmargin=0.75cm]
    \item The system is unknown ahead of time to the learner, nor is a stabilizing controller given. 
    
    \item The learner does not know the cost functions in advance. They can be adversarially chosen.
    
    \item The disturbances are not assumed to be stochastic, and can be adversarially chosen. 
\end{enumerate}

\paragraph{Robust and Optimal Control.} When the underlying system is known to the learner, the noise is stochastic, and costs are  known ahead of time, it is theoretically possible to compute the  optimal controller a priori. In the LQR setting, the costs are quadratic, dynamics are linear, and it follows from the Bellman equations that the optimal policy for infinite horizon is linear  $u_t = K x_t$, where $K$ is the solution to the algebraic Ricatti equation \citep{Stengel1994OptimalCA, kemin,Bert17}. Control which is robust to worst-case noise, in a min-max sense, is formulated in the framework of $H_\infty$ control, see e.g. the text by \cite{kemin}.

\paragraph{Online Control.}  
Recent literature stemming from the machine learning community considers the online LQR setting \citep{csaba,dean2018regret,mania2019certainty,cohen2018online}, where the noise remains i.i.d. Gaussian, but the performance metric is regret instead of cost. Recent algorithms in \citep{mania2019certainty,cohen19b,cohen2018online} attain $\sqrt{T}$ regret, with polynomial runtime and dependence on relevant problem parameters in the regret. This was improved to $O(\poly(\log T))$ by \cite{log_regret} for strongly convex costs.  Regret bounds for partially observed systems are studied in \citep{anima1,anima2,anima3}, and the most recent bounds are in \citep{simchowitz19a}. All the above results assume the learner is given a stabilizing controller. 

Black-box control of an unknown LDS was studied by \cite{csaba} and $\sqrt{T}$ regret was obtained, though the algorithm is inefficient in the sense that it may take  exponential running time in the worst case.  In contrast, our algorithm runs in polynomial time, and our setting permits non-i.i.d., and even adversarial noise sequences, and adversarial loss functions. 

Regret {\bf lower bounds} for online LQR were studied by \cite{tomerlowerbound} and \cite{simchowitz2020naive}, who show polynomial lower bounds in terms of the parameter $T$. In comparison, our lower bound is exponential, and in terms of the system dimension rather than time.

Concurrently and independently, recent work by \citet{lale2020explore} considers the black-box online LQR setting and obtain $\tilde{O}(2^{\mL} + \sqrt{T})$ regret under the weaker condition of stabilizability. However, their setting is restricted to stochastic, rather than adversarial, noise and quadratic, rather than general, cost functions. 

\paragraph{Nonstochastic Control: } Moving away from stochastic noise, the nonstochastic control problem for linear dynamical systems was posed in \citep{agarwal2019online} to capture more robust online control (see survey  \citep{hazan2020lecture}). In this setting, the controller has no knowledge of the system dynamics or the adversarial noise sequence. The controller generates controls $u_t$ at each iteration to minimize regret over sequentially revealed adversarial convex cost functions, against all disturbance-action control policies. If a strongly stable controller is known, \cite{hazan2019nonstochastic} give an algorithm that achieves $\tilde{O}(poly(\mL, \kappa^*)T^{2/3})$ regret, where $\mL$ is an upper bound on the system's natural parameters and $\kappa^*$ is the controllability parameter of the stabilized system, as formalized in Section \ref{sec:notations}. This was recently extended in \citep{simchowitz2020improper} to partially observed systems, and better bounds for certain families of loss functions with semi-adversarial noise. In \citep{simchowitz2020making}, $\tilde{O}(\poly(\mL, \tilde{\kappa})\sqrt{T})$ regret was obtained for the nonstochastic LQR problem, where $\tilde{\kappa}$ is an upper bound on the stability parameters of the recovered controllers, see Section  \ref{sec:nonstochastic}. However, all of these works assume that a stabilizing controller is given to the learner, and are {\bf not black-box} as per our definition.

\paragraph{Identification and Stabilization of Linear Systems: } If the system has stochastic noise, the least squares method can be used to identify the dynamics in the partially observable and fully observable settings \citep{oymak, simchowitz18a, sarkar19a}. Using this method of system identification, recent work by \cite{stab} finds a stabilizing controller in finite time. However, no explicit bounds were given on the cost or the number of total iterations required to identify the system to sufficient accuracy.  Moreover, least squares can lead to inconsistent solutions under adversarial noise. The algorithm by \cite{simchowitz19a} tolerates adversarial noise, but the guarantees only hold for stable systems. 

In contrast, our paper provides explicit finite-time bounds for optimally controlling the system even in the presence of adversarial noise. Our results do not assume stability of the system (spectral radius bounded by 1), but the weaker condition of controllability. It remains open to relax this assumption even further, to that of stabilizability in the nonstochastic black-box model.

\section{Setting and Background}
To enable the analysis of non-asymptotic regret bounds,  we consider regret minimization against the class of strongly stable linear controllers.  The notion of strong stability was formalized in \citep{cohen2018online} to quantify the rate of convergence to the steady-state distribution. 

\begin{definition}[Strong Stability]
$K$ is a $(\kappa, \gamma)$ strongly stable controller for $(A, B)$ if $\|K\|\le \kappa$, and there exist matrices $H$, $L$ such that $A + BK = HLH^{-1}$, and $\|H\|\|H^{-1}\| \le \kappa$, $\|L\|\le 1-\gamma$. 
\end{definition}
The regret definition in Section \ref{sec:results} is meaningful only when the comparator set $\Pi$ is non-empty. As shown in \citep{cohen2018online}, a system $(A, B)$ has a strongly stable controller if it is strongly controllable. This notion is formalized in the next definition. 
\begin{definition} [Strong Controllability]\label{def:strong-controllability}
Given a system $(A, B)$, let $C_k$ denote
$$
C_k = [B\ AB\ A^2B\ \cdots A^{k-1}B] \in \mathbb{R}^{d_x\times kd_u}.
$$
Then $(A, B)$ is $(k, \kappa)$ strongly controllable if $C_k$ has full row-rank, and $\|(C_kC_k^\top)^{-1}\| \le \kappa$. 
\end{definition}
\begin{assumption} \label{a:controllable}
The system $(A, B)$ is $(k, \kappa)$ strongly controllable for $\kappa \ge 1$, and $\|A\|, \|B\|\le \beta$ for some $\beta \ge 1$. 
\end{assumption}
Under Assumption 1, the noiseless dynamical system $x_{t+1} = Ax_t + Bu_t$ starting from $x_1$ can be driven to the zero state in $k$ steps. Furthermore, Lemma B.4 in \citep{cohen2018online} gives an upper bound on the reset cost, defined as $\sum_{t=1}^{k} \|x_t\|^2 + \|u_t\|^2$. In Section \ref{sec:sdp} we show that a bounded reset cost implies the existence of a strongly stable controller. As a consequence of the Cayley-Hamilton theorem, a controllable system's controllability index $k$ is at most $d_x$.
\noindent
Finally we make the following mild assumptions on the noise sequence and the cost functions. 
\begin{assumption}
The noise sequence is bounded such that $\|w_t\|\le 1$ for all $t$.
\end{assumption}

\begin{assumption} \label{a:lipschitz}
The cost functions are convex, and for all $x, u$ such that $\|x\|, \|u\|\le D$, $\|\nabla_{(x, u)} c_t(x, u)\|\le GD$. Without loss of generality, assume $c_t(0, 0) = 0$.
\end{assumption}

\subsection{Notations} \label{sec:notations}
Inspired by the convention from the theory of Linear Programming \citep{nemi}, we use $\mL$ to denote an upper bound on the natural parameters, which we interpret as the complexity of the system, i.e. 
$$ \mL = kd_u + d_x + G + \beta + \kappa,\ \text{where}$$ 
\begin{itemize}

    \item $\kappa, k$ are the controllability parameter and controllability index of the true system, respectively.
    \item $d_x, d_u$ are the dimension of the states $x_t \in \reals^{d_x}$ and dimension of the controls $u_t \in \reals^{d_u}$.
    \item $G$ is an upper bound on the Lipschitz constant of the cost functions $c_t$.
    \item $\beta$ is an upper bound on the spectral norm of system dynamics $A, B$.
\end{itemize}
Given a $(\tilde{\kappa}, \tilde{\gamma})$ strongly stable controller $K$, we denote $\kappa^*$ as the upper bound on the controllability parameter of the stabilized system $(A+BK, B)$, and $\tilde{\kappa}$. We henceforth prove an upper bound on $\kappa^*$ for the controller we recover, and show in Section \ref{sec:nonstochastic} that $\kappa^* \le \poly( \kappa, \beta^{k},d_x)$. We use $\tilde{O}$ to denote bounds that hold with probability at least $1-\delta$, and omit the $\log(\delta^{-1})$ factor.

\subsection{Disturbance-action Controller} \label{sec: dac}
In the canonical parameterization of the nonstochastic control problem, the total cost of a linear controller $J(K)$ is not convex in $K$. This problem is solved by considering a class of controllers called Disturbance-action Controllers (DACs) \citep{agarwal2019online}, which executes controls that are linear in past noises. The total cost of DACs is convex with respect to their parameters, and the cost of any strongly stable controller can be approximated by this class of controllers. Techniques in online convex optimization can then be used on this convex re-parameterization of the nonstochastic control problem. It is shown in \citep{hazan2019nonstochastic} that for an unknown system $(A, B)$ and a known $(\kappa, \gamma)$ strongly stable controller $K$, a DAC can achieve sublinear regret against all such controllers parametrized by $(K', M)$ where $K'$ is $(\kappa, \gamma)$ strongly stable.
\begin{definition}[Disturbance-action Controllers]\label{def:dac}
A disturbance-action controller with parameters $(K, M)$ where $M = [M^0, M^1, \ldots, M^{H-1}]$ outputs control $u_t$ at state $x_t$,

$$ u_t = Kx_t + \sum_{i=1}^H M^{i-1} w_{t-i}. $$
\end{definition}

DACs also include the class of Linear Dynamic Controllers (LDCs). LDCs are the state of the art in control, and is a generalization of static feedback controllers. Both $\H_2$ and $H_\infty$ optimal controllers under partial observation can be well-approximated by LDCs.

\begin{definition}[Linear Dynamic Controllers] \label{def:ldc}
A linear dynamic controller $\pi$ is a linear dynamical system $(A_\pi, B_\pi, C_\pi, D_\pi)$ with internal state $s_t\in \mathbb{R}^{d_\pi}$, input $x_t\in \mathbb{R}^{d_x}$ and output $u_t\in\mathbb{R}^{d_u}$ that satisfies
$$
s_{t+1} = A_\pi s_t + B_\pi x_t,\ \ u_t = C_\pi s_t + D_\pi x_t.
$$
\end{definition}

\subsection{SDP Relaxation for LQ Control} \label{sec:sdp}
In Linear Quadratic control the cost functions are known ahead of time and fixed,
$$
c_t(x, u) = x^\top Qx + u^\top R u,
$$
and the noise is i.i.d., $w_t\sim N(0, W)$. Given an instance of the LQ control problem defined by $(A, B, Q, R, W)$, the learner can obtain a strongly stable controller by solving the SDP relaxation for minimizing steady-state cost, proposed in \citep{cohen2018online}. For $\nu > 0$, the SDP is given by
\begin{equation*}
\begin{aligned}
& \underset{}{\text{minimize}}
& & J(\Sigma) = \begin{pmatrix}
Q & 0\\
0 & R
\end{pmatrix} \bullet \Sigma \\
& \text{subject to}
& & \Sigma_{xx} = \begin{pmatrix}
A & B
\end{pmatrix}\Sigma \begin{pmatrix}
A & B
\end{pmatrix}^\top + W, \ \ \Sigma = \begin{pmatrix}
\Sigma_{xx} & \Sigma_{xu}\\
\Sigma_{xu}^\top & \Sigma_{uu}
\end{pmatrix}.\\
& & & \Sigma \succeq 0,\ \Tr(\Sigma) \le  \nu.
\end{aligned}
\end{equation*}
Indeed, a strongly stable controller can be extracted from any feasible solution to the SDP, as guaranteed by the following lemma.
\begin{lemma}[Lemma 4.3 in \citep{cohen2018online}] Assume that $W\succeq \sigma^2 I$ and let $\kappa = \sqrt{\nu}/\sigma$. Let $\Sigma$ be any feasible solution for the SDP, 
then the controller $K = \Sigma_{xu}^\top \Sigma_{xx}^{-1}$ is $(\kappa, 1/2\kappa^2)$ strongly stable. \end{lemma}

\paragraph{Existence of Strongly Stable Controllers} Suppose a noiseless system $x_{t+1} = Ax_t + Bu_t$ can be driven to zero in $k$ steps with resetting cost $C\|x_1\|^2$. Thereom B.5 in \citep{cohen2018online} suggests that the SDP for the noisy system $x_{t+1} = Ax_t + Bu_t + w_t$ with $w_t\sim N(0, W)$ and $\nu = C\cdot \Tr(W)$ is feasible. Taking $W = I$, the system $(A, B)$ has a $(\sqrt{Cd_x}, 1/(2Cd_x))$ strongly stable controller. Lemma B.4 in \cite{cohen2018online} shows that under Assumption 1, $C = 3\kappa^2k^2\beta^{6k}$.


\ignore{
A linear dynamical system is given by
$$ x_{t+1} = A x_t + B u_t + \eps_t .$$
In online control a sequence of adversarially chosen convex cost functions are denoted $c_t(x_t,u_t)$. A linear controller sets controls of  the form 
 $u_t = K x_t $.

A non-stochastic contoller applies to LDS with adversarial cost sequence and perturbations. The regret for NSC controller $\A$ is defined to be 
$$  \regret(\A) = \E \left[ \sum_t c_t(x_t, u_t ) - \min_K \sum_t c_t(\hat{x}_t , K \hat{x}_t) \right], $$
where $u_t$ is the controller output at time $t$, i.e. 
$$ u_t = \A(x_{1:t},\mathcal{I}) . $$
Here $\mathcal{I}$ is any information known to the controller, such as the system or approximation thereof. 
}

\section{Algorithm and main theorem}
Now we describe our main algorithm for the black-box control problem, Algorithm
\ref{alg1_md}. Overall we use the explore-then-commit strategy, and split the algorithm into three phases. In phase 1, we identify the underlying system dynamics to within some accuracy with large controls. In phase 2, we extract a strongly stable controller for the estimated system using the SDP in Section \ref{sec:sdp}, and show that it is also strongly stable for the true system. We then alleviate the effects of using large controls by decaying the system to a state with constant magnitude. Finally in phase 3, we invoke Algorithm 1 in \cite{hazan2019nonstochastic} or Algorithm 3 in \citet{simchowitz2020making} to achieve sublinear regret. 

\begin{algorithm}
\caption{Nonstochastic Control with Black-box Access}
\label{alg1_md}
\begin{algorithmic}[1]
\STATE Input: horizon $T$, $k, \kappa$ such that the system $(A, B)$ is $(k, \kappa)$ strongly controllable, $\beta \ge 1$ such that $\|A\|, \|B\|\le \beta$. 
\STATE Set $\kappa' = \sqrt{Cd_x}$, $\gamma' = 1/(2\kappa'^2)$, where $C = 3\kappa^2k^2\beta^{6k}$.
\STATE \underline{Phase 1:} \textbf{Black-box System Identification}
\STATE Set $\eps = \frac{\gamma'^2}{10^5d_x^2\kappa'^8}$, $\lambda = 8\beta$.
\STATE $(\hat{A},\hat{B}) \leftarrow \text{AdvSysId}(\eps, \lambda, x_1, k, \kappa)$ for $T_1 =  d_u(k+1)+1 $ rounds. 
\STATE \underline{Phase 2:} \textbf{Stable Controller Recovery}
\STATE $\hat{K} \leftarrow \text{ControllerRecovery}(\hat{A}, \hat{B}, \eps, \kappa', \gamma')$, set $\tilde{\kappa} =\frac{2\kappa'^2d_x^{1/2}}{\gamma'^{1/2}},\ \tilde{\gamma} = \frac{\gamma'}{16d_x\kappa'^4}.$
\STATE Execute $\hat{K}$ for $T_2=\max\{\frac{\ln(\tilde{\gamma} \|x_{T_1}\|)}{\tilde{\gamma}}, 0\}$ rounds .
\STATE \underline{Phase 3:} \textbf{Nonstochastic Control}
\STATE Set $\kappa^* =  4\tilde{\kappa}^2k^2\beta^{2k}\kappa$, $W =2\kappa^*/\tilde{\gamma}$.
\STATE General convex costs: call Algorithm 1 in \cite{hazan2019nonstochastic} with inputs $\hat{K}$, $\kappa^*$, $\tilde{\gamma}$, $W$ for $T - T_1 - T_2$ rounds.
\STATE Strongly convex costs: call Algorithm 3 in \citet{simchowitz2020making} for $T - T_1 - T_2$ rounds.
\end{algorithmic}
\end{algorithm}

Our main theorem below is stated using asymptotic notation that hides constants independent of the system parameters, and uses $\mL$ for an upper bound on the system parameters as defined in section \ref{sec:notations}. Exact specification of the constants appear in the proofs. 
\begin{theorem}\label{thm:main}
Under Assumptions 1, 2, 3, with high probability the regret of Algorithm $\ref{alg1_md}$ is at most
$$ \regret_T(\A_1) \leq 2^{O(\mL \log \mL)} + \tilde{O}(\poly(\mL, \kappa^*) T^{2/3} ) .$$ If the loss functions are in addition $\alpha$-strongly convex, and without loss of generality assuming $\tilde{\kappa} \ge \tilde{\gamma}^{-1}$, the regret of Algorithm $\ref{alg1_md}$ is at most
$$ \regret_T(\A_1) \leq 2^{O(\mL \log \mL)} + \tilde{O}(\poly(\mL, \tilde{\kappa}, \alpha^{-1}) \sqrt{T} ) .$$
This is composed of 
\begin{enumerate}[leftmargin=0.8cm]
\setlength\itemsep{0.1em}
    \item Phase 1: after $T_1$ rounds we have $\|x_{T_1}\|^2 \le 2^{O(\mL \log \mL)} $. The total cost is at most 
    $  2^{O(\mL \log \mL)} $.
     \item
    Phase 2: Computing $\hat{K}$ has zero cost. Decaying the system has total cost at most $O\big(G \tilde{\kappa}^4\|x_{T_1}\|^3\tilde{\gamma}^{-3}\big), $ where
    $\tilde{\kappa}$, $\tilde{\gamma}$ are as defined in the algorithm. 
    This phase has total cost $2^{O(\mL \log \mL)}$.
    \item Phase 3: Nonstochastic control with a known strongly stable controller for general convex costs incurs regret at most $\tilde{O}(\poly(\mL, \kappa^*)(T-T_1-T_2)^{2/3})$, with high probability. If the cost functions are $\alpha$-strongly convex, with high probability the regret is bounded by 
     
     $\tilde{O}(\poly(\tilde{\kappa}, \mL, \alpha^{-1})\sqrt{T-T_1-T_2})$.
\end{enumerate}
\end{theorem}


\section{Analysis Outline}
We provide an outline of our regret analysis in this section. Formal statements are in the appendix.

\subsection{Black-box system identification } \label{sec:phase1}
In this phase we obtain estimates of the system $\hat{A}, \hat{B}$ without knowing a stabilizing controller. Recall the definition of $C_k=[B, AB, \ldots, A^{k-1}B]$, and let $Y = [AB\ A^2B\ \cdots\ A^kB]$. The procedure AdvSysId (Algorithm \ref{alg:sysid_bias}) consists of two steps. In the first step, we estimate each $A^jB$ for $j = 0\ldots, k$ (in particular we obtain $\hat{B}$ close to $B$), and guarantee that $\|C_k - C_0\|_F$, $\|Y - C_1\|_F$ are small. In the second step, we take $\hat{A}$ to be the solution to the system of equations in $X$: $XC_0 = C_1$.

For the first step, the algorithm estimates matrices $A^jB$ by using controls that are scaled standard basis vectors once every $k+1$ iterations, and using zero controls for the iterations in between. The state evolution satisfies
\[ x_{t+1} = A^t x_1 + \sum_{i=1}^t (A^{t-i}Bu_i + A^{t-i}w_i).\]
Intuitively, we choose scaling factors $\xi_i$ such that $j$ iterations after a non-zero control $\xi_i\cdot e_i$ is used, the state is dominated by $\xi_i A^{j-1}Be_i$, the scaled $i$-th column of $A^{j-1}B$. In the algorithm $\hat{M}_j$ is the concatenation of estimates for $A^jBe_i$, and we concatenate the $\hat{M}_j$'s to obtain $C_0, C_1$. We show in Lemma \ref{lem:frob} that $\|\hat{M}_j - A^jB\|_F \le O(d_u^2k\lambda^{2k}\eps_0)$, which implies the closeness of $C_0, C_1$ to $C_k, Y$, respectively. 

Under the assumption that $(A, B)$ is $(k, \kappa)$ strongly controllable, $A$ is the unique solution to the system of equations in $X$: $XC_k = Y$. By perturbation analysis of linear systems, the solution to the system of equations $XC_0 = C_1$ is close to $A$, as long as $\|C_0 - C_k\|_F, \|C_1 - Y\|_F$ are sufficiently small. By our choice of $\eps_0$, we conclude that $\|\hat{A} - A\|\le \eps$, $\|\hat{B} - B\|\le \eps$. Lemma \ref{lem:cost_1} shows that the total cost of this phase is bounded by $2^{O(\mL\log \mL)}$.
\begin{algorithm}[ht]
\caption{AdvSysId }
\label{alg:sysid_bias}
\begin{algorithmic}[1]
\STATE Input: accuracy parameter $\eps < 1/2$, $\|x_1\|\le 1$. Let $\lambda \ge 1$ be such that $ \|A\|,\|B\| \leq \frac{1}{4} \lambda - 1$, $(k, \kappa)$ such that the system $(A, B)$ is $(k, \kappa)$ strongly controllable.
\STATE Set $\eps_0 = \frac{\eps}{10^2d_u^2k^2\lambda^{3k}d_x\kappa^{1/2}}$.
\FOR{$t = 1, \ldots, (k+1)d_u$}
\STATE observe $x_t$.
\IF {$t = 1 \pmod{k+1}$}
\STATE Let $i = (t-1)/(k+1) + 1$.
\STATE control with $u_t = \xi_i \cdot e_i$ for $\xi_i =  \lambda^{t-1}\eps_0^{-i}$, where $e_i$ is the $i$-th standard basis vector.
\ELSE
\STATE control with $u_t = 0.$
\ENDIF
\STATE pay cost $c_t(x_t,u_t)$.
\ENDFOR
\STATE For $0\le j\le k$, $1\le i\le d_u$, define $l(i, j) = (i-1)(k+1)+j+2$. Let $x_i^j = x_{l(i, j)}$. Construct
 $$\hat{M_j} =  [\frac{x_1^j}{\xi_1}\  \frac{x_2^j}{\xi_2}\ \cdots \ \frac{x_{d_u}^j}{\xi_{d_u}}]\in \mathbb{R}^{d_x\times d_u}.  $$
\STATE Define $C_0 = [\hat{M}_0\ \hat{M}_1\ \cdots\ \hat{M}_{k-1}]$, $C_1 = [\hat{M}_1\ \hat{M}_2\ \cdots\ \hat{M}_{k}]\in \mathbb{R}^{d_x\times d_uk}$. 
\STATE Output $\hat{A} = C_1C_0^\top(C_0C_0^\top)^{-1}$, $\hat{B} = \hat{M}_0$.
\end{algorithmic}
\end{algorithm}

\subsection{Computing a stabilizing controller} \label{sec:phase2}
The goal of phase 2 is to recover a strongly stable controller from system estimates obtained in phase 1 by solving the SDP presented in Section \ref{sec:sdp}. The key to our task is setting the trace upper bound $\nu$ appropriately, so that the SDP is feasible and the recovered controller is strongly stable even for the original system. We justify our choice of $\nu$ in Lemma \ref{lem:sdp}, and show that by our choice of $\eps$, $\hat{A}, \hat{B}$ are sufficiently accurate and $\hat{K}$ is $(\tilde{\kappa}, \tilde{\gamma})$ strongly stable for the true system. We remark that \cite{simchowitz2020naive} is an alternative procedure for recovering $K$, given system estimates.

\begin{algorithm}[ht]
\caption{ControllerRecovery}
\label{alg:k_recovery}
\begin{algorithmic}[1]
\STATE Input: $\kappa'$, $\gamma'$ such that there exists $K$ that is $(\kappa', \gamma')$ strongly stable for $(A, B)$; accuracy parameter $\eps$, and $\hat{A}$, $\hat{B}$ such that $\|A - \hat{A}\| \le \eps$, $\|B - \hat{B}\| \le \eps$.
\STATE Set $\nu = \frac{2\kappa'^4d_x}{\gamma' - 2\eps\kappa'^2}$.
\STATE Solve the following SDP:
\begin{equation*}
\begin{aligned}
& \underset{}{\text{minimize}}
& & 0 \\
& \text{subject to}
& & \Sigma_{xx} = \begin{pmatrix}
\hat{A} & \hat{B}
\end{pmatrix}\Sigma \begin{pmatrix}
\hat{A} & \hat{B}
\end{pmatrix}^\top + I,\ \text{where}\\
& & & \Sigma = \begin{pmatrix}
\Sigma_{xx} & \Sigma_{xu}\\
\Sigma_{xu}^\top & \Sigma_{uu}
\end{pmatrix},\ \Sigma \succeq 0,\ \Tr(\Sigma) \le  \nu.
\end{aligned}
\end{equation*}
\STATE Denote a feasible solution as $\hat{\Sigma} = \begin{pmatrix}
\hat{\Sigma}_{xx} & \hat{\Sigma}_{xu}\\
\hat{\Sigma}_{xu}^\top & \hat{\Sigma}_{uu}
\end{pmatrix},
$ return
$\hat{K} = \hat{\Sigma}_{xu}^\top\hat{\Sigma}_{xx}^{-1}.$
\end{algorithmic}
\end{algorithm}
\subsubsection{Decaying the system}

In phase 1 the algorithm uses large controls to estimated the system, and after $T_1$ iterations the state might have an exponentially large magnitude. Equipped with a strongly stable controller, we decay the system so that the state has a constant magnitude before starting phase 3. We show in Lemma \ref{lem:decay} that following the policy $u_t = \hat{K}x_t$ for $T_2$ iterations decays the state to at most $2\tilde{\kappa}/\tilde{\gamma}$ in magnitude.

\subsection{Nonstochastic control } \label{sec:nonstochastic}

Given a $(\tilde{\kappa}, \tilde{\gamma})$ strongly stable controller $\hat{K}$ for the true system, if the costs are general convex functions, we run Algorithm 1 in \citep{hazan2019nonstochastic} (Algorithm \ref{alg:gpc_md} in the appendix) which achieves sublinear regret. By Lemma \ref{lem:controllable}, the system $(A+B\hat{K}, B)$ is $(k, 4\tilde{\kappa}^2k^2\beta^{2k}\kappa)$ strongly controllable.

If we start Algorithm \ref{alg:gpc_md} from $t=T_1 + T_2$, the setting is consistent with the nonstochastic control setting where the noise is bounded by $\|x_{T_1+T_2}\|$, and with total iteration number $T-T_1 - T_2$.
By Theorem 12 in \cite{hazan2019nonstochastic}, setting $\kappa^* =  4\tilde{\kappa}^2k^2\beta^{2k}\kappa$, $W =2\kappa^*/\tilde{\gamma}$, and noticing that $\tilde{\gamma}^{-1}=\poly(\kappa^*)$, with high probability, our total regret is at most
$
\tilde{O}(\poly(\kappa^*, k, d_x, d_u, G)T^{2/3}).$

If the cost functions are $\alpha$-strongly convex, we use Algorithm 3 in \citet{simchowitz2020making}. Note that this algorithm does not need controllability assumptions on the system. By Theorem 3.2 in \citet{simchowitz2020making}, and without loss of generality assuming $\tilde{\kappa} \ge \tilde{\gamma}^{-1}$, with high probability the total regret of this phase is bounded by 
$
\tilde{O}(\poly(\tilde{\kappa}, \beta, d_x, d_u, G, \alpha^{-1})\sqrt{T}).
$
\section{Lower Bound on Black-box Control}

In this section we prove that with high probability, any randomized black-box control algorithm incurs a loss which is exponential in the system dimension, even for noiseless LTI systems. Our lower bound is partially based on the construction in \citet{gradientlowerbound}. In addition, we provide a lower bound for deterministic black-box control algorithms in Appendix \ref{a:lower_bound_det}, which can be extended to degenerate systems and has improved constants. We first define the relevant concepts.

\begin{definition}[Black-box Control Algorithm] \label{defn:bb-control-alg}
A randomized black-box control algorithm $\A$ has random string $\sigma_t$ and outputs a control $u_t$ at each iteration $t$, where $u_t$ is a function of past information and the random string, i.e. $u_t = \A(x_1,...,x_t, c_1, \ldots, c_t, u_1, \ldots, u_{t-1}, \sigma_t)$. 
\end{definition}

\begin{definition}[Control Problem Instance]
\label{defn:control-problem}
An instance of a control problem is defined by a noiseless system $(A, B)$, an initial state $x_1$, and a sequence of oblivious convex cost functions $\{c_t\}$.
\end{definition}

\begin{theorem} \label{thm:lowerbound_rand}
Let $\A$ be a randomized control algorithm as per Definition \ref{defn:bb-control-alg}. Then there exists a control problem instance with system dimension $d_x$, where the system is stabilizable and $(1, 1)$-strongly controllable, such that with $T = d_x/8$, with probability at least $1 - \exp(-\frac{d_x}{100})$, we have 
$$ \regret_T(\A)\ge 2^{\Omega(\mL)} . $$
\end{theorem}

\begin{proof}
The proof follows from the intuition that a matrix with i.i.d. random Gaussian entries is rotation invariant, and therefore for such a matrix of dimension $d$, one needs to observe at least $O(d)$ matrix-vector products to "learn" it.
\paragraph{The construction.} Let $x_1 = e_1$, and $c_t(x, u) = \|x\|^2 + \|u\|^2$ for all $t$. Consider the noiseless system $x_{t+1} = Ax_t + u_t $, where $A\sim N(d_x,d_x,\frac{\gamma}{d_x})$, for $N(a,b,c)$ describing a distribution over matrices of dimension $a \times b$, with each entry being i.i.d. normally distributed with mean zero and variance $c$. Note that this system is $(1, 1)$-strongly controllable, and $-A$ is a stabilizing controller that gives constant regret. Moreover, with high probability, the system has bounded size: by Corollary 35 in \citep{vershynin2011introduction}, with probability at least $1-2\exp(-\frac{d_x}{2})$, $\|A\| \le 3\sqrt{\gamma}$. Let $\mL(A)$ denote the system upper bound of the control problem instance defined by our choice of $x_1, \{c_t\}$, and $A$. Under this event, we have $\mL(A) \le 2d_x + 4 + 1 + 3\sqrt{\gamma}\le 4d_x$ for $\gamma = 40$ and $d_x$ large.
\\

In this proof, we first consider deterministic black-box control algorithms and show that with high probability, the total cost of any deterministic control algorithm is exponential in the system dimension. Then we treat a randomized algorithm as a distribution over deterministic algorithms, and use a probabilistic argument to show that there exists a hard control problem for every randomized algorithm. We define the equivalent information model for deterministic control algorithms below.

\paragraph{Information model.} At every iteration the controller observes $x_t$, then computes $u_t$ as a deterministic function of $x_1, x_2, \ldots, x_t$, $u_1, u_2, \ldots u_{t-1}$, and then observes $x_{t+1} = Ax_t + u_t$. Without loss of generality, we can assume that the controller also observes $Au_1, \ldots, Au_{t-1}$ before computing $u_t$, but does not act on this information. Then the states of our dynamical system can be seen as queries and observations in the following information model: a player makes deterministic, adaptive queries defined by vectors $x_1, u_1, x_2, u_2, \ldots, x_{t-1}, u_{t-1}, x_t, u_t$ and observes $Ax_1, Au_1, \ldots, Ax_{t-1}, Au_{t-1}, Ax_t, Au_t$. Each pair of queries $x_t, u_t$ are deterministic functions of previous queries and observations: $x_1, \ldots, x_{t-1},$ $u_1, \ldots, u_{t-1}, Ax_1, \ldots, Ax_{t-1}. Au_1, \ldots, Au_{t-1}$. Note that even though $u_t$ can depend on $x_t$, we have $x_t = Ax_{t-1}+u_{t-1}$, so without loss of generality we can assume $u_t$ only depends on previous queries and observations. \\

Under this information model, for every $x_t$, there exists a subspace $(V_{t-1}^\perp)^\top$ such that $(V_{t-1}^\perp)^\top x_t$ has a random component. Importantly, the subspace only depends on the queries and observations seen so far and not on any future information. Further, we can show that with high probability, the magnitude of the random component grows exponentially with time. The following two lemmas make this precise.

 \begin{lemma} \label{lem:distr}
Let $T = d_x/8$. There exists a sequence of orthonormal matrices $V_1, \ldots, V_{T}$, such that for $t\in [T]$, $V_t$ only depends on $x_1, \ldots, x_t, u_1, \ldots, u_t$ and they satisfy the following condition: \\Let $r_t$ denote the rank of $\text{span}(x_1, \ldots, x_t,$ $ u_1, \ldots, u_t)$, and denote the first $r_t$ columns of $V_t$ as $V_t^{\|}$, and the last $d-r_t$ columns of $V_t$ as $V_t^\bot$.
Let 
$h_t = (V_{t-1}^{\bot})^\top x_t$, then for all $t\in[T]$, conditioned on $x_1, x_2, \ldots, x_t,$ $u_1, u_2, \ldots, u_{t}, Au_1, \ldots, Au_{t-1}$, we have $(V_{t}^\bot)^\top x_{t+1} = c_t + z_t$, where the coordinates of $z_t$ are i.i.d. normally distributed, i.e. $z_t(i) \sim N(0, \frac{\gamma \|h_t\|^2}{d})$.
\end{lemma}

\begin{lemma}\label{lem:lower_bound_magnitude}
Let $V_1, \ldots, V_T$ be as in Lemma \ref{lem:distr}, and $T = d_x/8$. Let $h_t = (V_{t-1}^{\bot})^\top x_t$, with probability at least $1-\exp(-\frac{d_x}{25})$, conditioned on $x_1, x_2, \ldots, x_t, u_1, u_2, \ldots, u_t$, $Au_1, \ldots, Au_{t-1}$, we have $\|(V_{t}^\bot)^\top x_{t+1}\|^2 \ge \frac{\gamma \|h_t\|^2}{20}$.
\end{lemma}

 Consider the construction of matrices $V_1, V_2, \ldots, V_T$ as in Lemma \ref{lem:distr}. Then conditioned on $x_1, u_1,$ we have $h_1 = (V_0^{\bot})^\top x_1 = x_1$, and $\|h_1\| = 1$. Here $u_1$ can depend on $x_1$ because $x_1$ is independent of $A$. By Lemma \ref{lem:lower_bound_magnitude}, for $t\le T$, conditioned on $x_1, \ldots, x_t$, $u_1, \ldots, u_t$, $Au_1, \ldots, Au_{t-1}$, with probability at least $1 - \exp(-\frac{d_x}{25})$, we have $\|h_{t+1}\|^2 \ge 2\|h_t\|^2$ with our choice of $\gamma$. Therefore, with probability at least $(1 - \exp(-\frac{d_x}{25}))^{T-1}$, $\|h_T\|^2 \ge 2^{T-1}$. Note that $\|x_T\|^2  = \|V_{T-1}x_T\|^2 \ge \|V_{T-1}^\bot x_T\|^2 = \|h_T\|^2 \ge 2^{d_x/8-1}$. Since for small $\eps$, we have $(1-\eps)^t \ge 1-2t\eps$, we have $(1 - \exp(-\frac{d_x}{25}))^{\frac{d_x}{8}-1} \ge 1-\frac{d_x}{4}\exp(-\frac{d_x}{25}) \ge 1 - \exp(-\frac{d_x}{50})$ for $d_x$ large. Therefore with high probability, the total cost of any deterministic black-box control algorithm $\A$ over $T$ iterations is at least $2^{d_x/8 - 1}$ by our choice of cost functions. Note that this result holds with any realization of $Au_1, Au_2, \ldots, Au_T$. Since there exists a stabilizing controller that incurs constant cost, we conclude that $\regret_T(\A)\ge 2^{\Omega(d_x)}$ with probability at least $1 - \exp(-\frac{d_x}{50})$, for any deterministic $\A$.
 
 Now we consider randomized control algorithms. For any randomized algorithm $\A_{\text{rand}}$, its randomness is independent of the distribution over the system, and can be considered as a random string whose value is chosen before the start of the algorithm. Let $\sigma_T$ denote the randomness of $\A_{\text{rand}}$ over $T$ iterations, and for any value $b_T$ of $\sigma_T$, let $\A_{\text{rand}}(b_T)$ denote the algorithm which is $\A_{\text{rand}}$ with $\sigma_T$ fixed to $b_T$. Then $\A_{\text{rand}}(b_T)$ is a deterministic algorithm. Let $\regret_T(\A_{\text{rand}}(b_t), A)$ denote the regret of $\A_{\text{rand}}(b_T)$ on the system $A$. We can write 
 \begin{align*}
 \Pr_{A, \sigma_T}[\regret_T(\A_{\text{rand}}, A)\ge 2^{\Omega(d_x)}] &=\sum_{b_T} \Pr[\sigma_T = b_T]\Pr_A[\regret_T(\A_{\text{rand}}(b_T), A) \ge 2^{\Omega(d_x)}]\\
 &\ge \min_{\A_{\text{det}}}\Pr_A[\regret_T(\A_{\text{det}}, A) \ge 2^{\Omega(d_x)}] \ge 1 - \exp(-\frac{d_x}{50}).
 \end{align*}
In addition to having regret exponential in the system dimension, we also need the size of the system to be bounded. Let $\D$ denote the distribution of $A$ conditioned on the event $\mathcal{E}: \mL(A) \le O(d_x)$. Since the event $\mathcal{E}$ happens with probablility at most $2\exp(-\frac{d_x}{2})$, we have 
 \begin{align*}
 \Pr_{\A_{\text{rand}}, A\sim \D}[\regret_T(\A_{\text{rand}}, A) \ge 2^{\Omega(d_x)}] &\ge  \Pr_{\A_{\text{rand}}, A}[\regret_T(\A_{\text{rand}}, A) \ge 2^{\Omega(d_x)}] - \Pr[\mathcal{E}]\\
 &\ge 1 - 2\exp(-\frac{d_x}{50}) \\
 &\ge 1 - \exp(-\frac{d_x}{100}).
 \end{align*}
It follows that there exist a system $A^\star$ with $\mL(A^\star) \le O(d_x)$, such that over the randomness of $\A_{\text{rand}}$, with probability at least $1 - \exp(-\frac{d_x}{100})$,
$
\regret_T(\A_{\text{rand}}, A^\star) \ge 2^{\Omega(\mL(A^\star))}.
$
\end{proof}

\section{Conclusion}
We present the first end-to-end, efficient black-box control algorithm for unknown linear dynamical systems in the nonstochastic control setting. This improves upon previous work in several dimensions: computational efficiency (previous methods were exponential time), robustness (tolerating adversarial noise), and generality (our algorithm permits general adversarial convex functions instead of restricting them to quadratic functions). 

The startup cost of our algorithm is exponential in the system dimension. However we show that this cost is nearly optimal by giving a novel lower bound for any randomized or deterministic black-box control algorithm. Combined with previous results, our algorithm applied to the nonstochastic online LQR setting achieves near optimal regret.

One intriguing open problem in our setting is whether or not it is possible to achieve our regret upper bound with control signals that are not exponential in the system parameters. As far as we know, this possibility remains open. 


\begin{ack}
We acknolwedge helpful discussions with Blake Woodworth and Max Simchowitz.
\end{ack}


\bibliographystyle{plainnat}
\bibliography{COLT/ref}

\begin{thebibliography}{30}
\providecommand{\natexlab}[1]{#1}
\providecommand{\url}[1]{\texttt{#1}}
\expandafter\ifx\csname urlstyle\endcsname\relax
  \providecommand{\doi}[1]{doi: #1}\else
  \providecommand{\doi}{doi: \begingroup \urlstyle{rm}\Url}\fi

\bibitem[Abbasi-Yadkori and Szepesv{\'a}ri(2011)]{csaba}
Yasin Abbasi-Yadkori and Csaba Szepesv{\'a}ri.
\newblock Regret bounds for the adaptive control of linear quadratic systems.
\newblock In \emph{Proceedings of the 24th Annual Conference on Learning
  Theory}, pages 1--26, 2011.

\bibitem[Agarwal et~al.(2019{\natexlab{a}})Agarwal, Bullins, Hazan, Kakade, and
  Singh]{agarwal2019online}
Naman Agarwal, Brian Bullins, Elad Hazan, Sham Kakade, and Karan Singh.
\newblock Online control with adversarial disturbances.
\newblock In \emph{International Conference on Machine Learning}, pages
  111--119, 2019{\natexlab{a}}.

\bibitem[Agarwal et~al.(2019{\natexlab{b}})Agarwal, Hazan, and
  Singh]{log_regret}
Naman Agarwal, Elad Hazan, and Karan Singh.
\newblock Logarithmic regret for online control.
\newblock In H.~Wallach, H.~Larochelle, A.~Beygelzimer, F.~d\textquotesingle
  Alch\'{e}-Buc, E.~Fox, and R.~Garnett, editors, \emph{Advances in Neural
  Information Processing Systems 32}, pages 10175--10184. Curran Associates,
  Inc., 2019{\natexlab{b}}.
\newblock URL
  \url{http://papers.nips.cc/paper/9207-logarithmic-regret-for-online-control.pdf}.

\bibitem[Bertsekas(2017)]{Bert17}
Dimitri~P. Bertsekas.
\newblock \emph{Dynamic Programming and Optimal Control}, volume~I.
\newblock Athena Scientific, Belmont, MA, USA, 4th edition, 2017.

\bibitem[Braverman et~al.(2020)Braverman, Hazan, Simchowitz, and
  Woodworth]{gradientlowerbound}
Mark Braverman, Elad Hazan, Max Simchowitz, and Blake Woodworth.
\newblock The gradient complexity of linear regression.
\newblock In Jacob Abernethy and Shivani Agarwal, editors, \emph{Proceedings of
  Thirty Third Conference on Learning Theory}, volume 125 of \emph{Proceedings
  of Machine Learning Research}, pages 627--647. PMLR, 09--12 Jul 2020.
\newblock URL \url{http://proceedings.mlr.press/v125/braverman20a.html}.

\bibitem[Cassel et~al.(2020)Cassel, Cohen, and Koren]{tomerlowerbound}
Asaf Cassel, Alon Cohen, and Tomer Koren.
\newblock Logarithmic regret for learning linear quadratic regulators
  efficiently, 2020.

\bibitem[Cohen et~al.(2018)Cohen, Hassidim, Koren, Lazic, Mansour, and
  Talwar]{cohen2018online}
Alon Cohen, Avinatan Hassidim, Tomer Koren, Nevena Lazic, Yishay Mansour, and
  Kunal Talwar.
\newblock Online linear quadratic control, 2018.

\bibitem[Cohen et~al.(2019)Cohen, Koren, and Mansour]{cohen19b}
Alon Cohen, Tomer Koren, and Yishay Mansour.
\newblock Learning linear-quadratic regulators efficiently with only $\sqrt{T}$
  regret.
\newblock In Kamalika Chaudhuri and Ruslan Salakhutdinov, editors,
  \emph{Proceedings of the 36th International Conference on Machine Learning},
  volume~97 of \emph{Proceedings of Machine Learning Research}, pages
  1300--1309, Long Beach, California, USA, 09--15 Jun 2019. PMLR.
\newblock URL \url{http://proceedings.mlr.press/v97/cohen19b.html}.

\bibitem[Dean et~al.(2018)Dean, Mania, Matni, Recht, and Tu]{dean2018regret}
Sarah Dean, Horia Mania, Nikolai Matni, Benjamin Recht, and Stephen Tu.
\newblock Regret bounds for robust adaptive control of the linear quadratic
  regulator.
\newblock In \emph{Advances in Neural Information Processing Systems}, pages
  4188--4197, 2018.

\bibitem[{Faradonbeh} et~al.(2019){Faradonbeh}, {Tewari}, and
  {Michailidis}]{stab}
M.~K.~S. {Faradonbeh}, A.~{Tewari}, and G.~{Michailidis}.
\newblock Finite-time adaptive stabilization of linear systems.
\newblock \emph{IEEE Transactions on Automatic Control}, 64\penalty0
  (8):\penalty0 3498--3505, 2019.

\bibitem[Hazan(2020)]{hazan2020lecture}
Elad Hazan.
\newblock Lecture notes: Computational control theory.
\newblock \url{https://sites.google.com/view/cos59x-cct/lecture-notes}, 2020.
\newblock [Online; accessed 15-Jan-2021].

\bibitem[Hazan et~al.(2020)Hazan, Kakade, and Singh]{hazan2019nonstochastic}
Elad Hazan, Sham Kakade, and Karan Singh.
\newblock The nonstochastic control problem.
\newblock In \emph{Algorithmic Learning Theory}, pages 408--421. PMLR, 2020.

\bibitem[Lale et~al.(2020{\natexlab{a}})Lale, Azizzadenesheli, Hassibi, and
  Anandkumar]{anima1}
Sahin Lale, Kamyar Azizzadenesheli, Babak Hassibi, and Anima Anandkumar.
\newblock Regret bound of adaptive control in linear quadratic gaussian (lqg)
  systems, 2020{\natexlab{a}}.

\bibitem[Lale et~al.(2020{\natexlab{b}})Lale, Azizzadenesheli, Hassibi, and
  Anandkumar]{anima2}
Sahin Lale, Kamyar Azizzadenesheli, Babak Hassibi, and Anima Anandkumar.
\newblock Logarithmic regret bound in partially observable linear dynamical
  systems, 2020{\natexlab{b}}.

\bibitem[Lale et~al.(2020{\natexlab{c}})Lale, Azizzadenesheli, Hassibi, and
  Anandkumar]{anima3}
Sahin Lale, Kamyar Azizzadenesheli, Babak Hassibi, and Anima Anandkumar.
\newblock Regret minimization in partially observable linear quadratic control,
  2020{\natexlab{c}}.

\bibitem[Lale et~al.(2020{\natexlab{d}})Lale, Azizzadenesheli, Hassibi, and
  Anandkumar]{lale2020explore}
Sahin Lale, Kamyar Azizzadenesheli, Babak Hassibi, and Anima Anandkumar.
\newblock Explore more and improve regret in linear quadratic regulators,
  2020{\natexlab{d}}.

\bibitem[Laurent and Massart(2000)]{laurent2000}
B.~Laurent and P.~Massart.
\newblock Adaptive estimation of a quadratic functional by model selection.
\newblock \emph{Ann. Statist.}, 28\penalty0 (5):\penalty0 1302--1338, 10 2000.
\newblock \doi{10.1214/aos/1015957395}.
\newblock URL \url{https://doi.org/10.1214/aos/1015957395}.

\bibitem[Mania et~al.(2019)Mania, Tu, and Recht]{mania2019certainty}
Horia Mania, Stephen Tu, and Benjamin Recht.
\newblock Certainty equivalent control of lqr is efficient.
\newblock \emph{arXiv preprint arXiv:1902.07826}, 2019.

\bibitem[Nemirovski(1994-1995)]{nemi}
Arkadi Nemirovski.
\newblock Lecture on information-based complexity of convex programming,
  1994-1995.

\bibitem[{Oymak} and {Ozay}(2019)]{oymak}
S.~{Oymak} and N.~{Ozay}.
\newblock Non-asymptotic identification of lti systems from a single
  trajectory.
\newblock In \emph{2019 American Control Conference (ACC)}, pages 5655--5661,
  2019.

\bibitem[Sarkar and Rakhlin(2019)]{sarkar19a}
Tuhin Sarkar and Alexander Rakhlin.
\newblock Near optimal finite time identification of arbitrary linear dynamical
  systems.
\newblock In Kamalika Chaudhuri and Ruslan Salakhutdinov, editors,
  \emph{Proceedings of the 36th International Conference on Machine Learning},
  volume~97 of \emph{Proceedings of Machine Learning Research}, pages
  5610--5618, Long Beach, California, USA, 09--15 Jun 2019. PMLR.
\newblock URL \url{http://proceedings.mlr.press/v97/sarkar19a.html}.

\bibitem[Simchowitz(2020)]{simchowitz2020making}
Max Simchowitz.
\newblock Making non-stochastic control (almost) as easy as stochastic, 2020.

\bibitem[Simchowitz and Foster(2020)]{simchowitz2020naive}
Max Simchowitz and Dylan~J. Foster.
\newblock Naive exploration is optimal for online lqr, 2020.

\bibitem[Simchowitz et~al.(2018)Simchowitz, Mania, Tu, Jordan, and
  Recht]{simchowitz18a}
Max Simchowitz, Horia Mania, Stephen Tu, Michael~I. Jordan, and Benjamin Recht.
\newblock Learning without mixing: Towards a sharp analysis of linear system
  identification.
\newblock In S\'ebastien Bubeck, Vianney Perchet, and Philippe Rigollet,
  editors, \emph{Proceedings of the 31st Conference On Learning Theory},
  volume~75 of \emph{Proceedings of Machine Learning Research}, pages 439--473.
  PMLR, 06--09 Jul 2018.
\newblock URL \url{http://proceedings.mlr.press/v75/simchowitz18a.html}.

\bibitem[Simchowitz et~al.(2019)Simchowitz, Boczar, and Recht]{simchowitz19a}
Max Simchowitz, Ross Boczar, and Benjamin Recht.
\newblock Learning linear dynamical systems with semi-parametric least squares.
\newblock In Alina Beygelzimer and Daniel Hsu, editors, \emph{Proceedings of
  the Thirty-Second Conference on Learning Theory}, volume~99 of
  \emph{Proceedings of Machine Learning Research}, pages 2714--2802, Phoenix,
  USA, 25--28 Jun 2019. PMLR.
\newblock URL \url{http://proceedings.mlr.press/v99/simchowitz19a.html}.

\bibitem[Simchowitz et~al.(2020)Simchowitz, Singh, and
  Hazan]{simchowitz2020improper}
Max Simchowitz, Karan Singh, and Elad Hazan.
\newblock Improper learning for non-stochastic control, 2020.

\bibitem[Stengel(1994)]{Stengel1994OptimalCA}
Robert~F. Stengel.
\newblock \emph{Optimal Control and Estimation}.
\newblock 1994.

\bibitem[Tu(2019)]{tu2019sample}
Stephen~Lyle Tu.
\newblock \emph{Sample Complexity Bounds for the Linear Quadratic Regulator}.
\newblock PhD thesis, UC Berkeley, 2019.

\bibitem[Vershynin(2011)]{vershynin2011introduction}
Roman Vershynin.
\newblock Introduction to the non-asymptotic analysis of random matrices, 2011.

\bibitem[Zhou et~al.(1996)Zhou, Doyle, and Glover]{kemin}
Kemin Zhou, John~C. Doyle, and Keith Glover.
\newblock \emph{Robust and Optimal Control}.
\newblock Prentice-Hall, Inc., USA, 1996.
\newblock ISBN 0134565673.

\end{thebibliography}

\newpage
\appendix
\section{Proofs for Section \ref{sec:phase1}}
In this section we present proofs for phase 1 of Algorithm 1. We show that the estimates $\hat{A}, \hat{B}$ satisfy $\|\hat{A} - A\| \le \eps, \|\hat{B} - B\|\le \eps$, and bound the total cost of this phase. We first bound the magnitude of states in each iteration to guide our choice of scaling factors $\xi_i$.
\begin{restatable}{claim}{magnitude}
\label{claim: magnitude_1}
In Algorithm \ref{alg:sysid_bias}, for all $t = 2, \ldots, (k+1)d_u$, let $j = t-2 \pmod{k+1}$, $i = (t-2-j)/(k+1)+1$, we have
$\|x_t\| \leq \lambda^{t-1}\eps_0^{-i}$.
\end{restatable}
\begin{proof}
This can be seen by induction. For our base case, consider $x_2$, where $i=1, j=0$. $\|x_2\| \le \|A\| + \|B\|\|u_1\| + 1\le \frac{1}{4}\lambda(1+\eps_0^{-1})+1 \le \lambda\eps_0^{-1}$. Assume $\|x_t\|\le \lambda^{t-1}\eps_0^{-i}$ for $t = (i-1)(k+1)+j+2$. If $j=k$, then $t = i(k+1) + 1$, $\|u_t\| = \lambda^{t-1}\eps_0^{-i-1}$, and $t+1 = i(k+1)+2$.
$$ \|x_{t+1}\| \leq \|A\| \|x_{t}\| + \|B\| \|u_{t}\| + \|w_{t}\| \leq \frac{1}{4} \lambda ( \lambda^{t-1}\eps_0^{-i} + \lambda^{t-1}\eps_0^{-i-1}) + 1 \leq \lambda^t \eps_0^{-i-1}.$$
Otherwise, we have $0\le j \le k-1$, and $u_t = 0$. Moreover, $t+1\in\{(i-1)(k+1)+3, \ldots,(i-1)(k+1)+2+k\}$. Therefore
$$
\|x_{t+1}\| \leq \|A\| \|x_{t}\| + \|B\| \|u_{t}\| + \|w_{t}\| \leq \frac{1}{4} \lambda^{t}\eps_0^{-i} + 1 \leq \lambda^t \eps_0^{-i}.
$$
\end{proof}

With appropriate choice of $\xi_i$, we ensure that $\hat{M_j}$ and $A^jB$ are close in the Frobenius norm. 
\begin{restatable}{lemma}{frob}
\label{lem:frob}
For $j = 0, \ldots, k$, $\hat{M}_j$ satisfies 
$$
\|\hat{M}_j - A^jB\|_F \le 3d_u^2k\lambda^{2k}\eps_0.
$$
In particular, $\|\hat{M}_0 - B\| \le 3d_u^2k\lambda^{2k}\eps_0\le 
\eps$.
\end{restatable}
\begin{proof}
Observe that by definition,
\begin{align*}
    x_{t+1} = A^tx_1 + \sum_{s=1}^t A^{t-s}Bu_s + A^{t-s}w_s.
\end{align*}
We have $\|A^tx_1 +  \sum_{s=1}^t A^{t-s}w_s\|\le \lambda^t + \sum_{s=1}^t\lambda^{t-s} \le (t+1)\lambda^t.$
Note that the magnitude of this term should be small once we normalize by $\xi_i$. Let $j = t-2 \pmod{k+1}$, $i = (t-2-j) / (k+1) + 1$, then $t = (i-1)(k+1)+j+2$. We proceed to bound $\|x_t/\xi_i- (A^jB)_i\|$, where $(A^jB)_i$ is the $i$th column of $A^jB$. The largest sum in $x_t$ can be analyzed as follows,
\begin{align*}
  \sum_{s=1}^{t-1} A^{t-1-s}Bu_s &= \sum_{s=1}^{(i-1)(k+1)+j+1} A^{(i-1)(k+1)+j+1-s}Bu_s \\
  &= \sum_{r=0}^{i-1}A^{(i-1-r)(k+1)+j}Bu_{r(k+1)+1}\\
  &= \sum_{r=0}^{i-1}A^{(i-1-r)(k+1)+j}B\xi_{r+1}e_{r+1}\\
  &=\sum_{r=1}^{i}\eps_0^{-r}\lambda^{(r-1)(k+1)}A^{(i-r)(k+1)+j}Be_{r}.
\end{align*}
Normalizing by the scaling factor,
\begin{align*}
    \frac{1}{\xi_i} \sum_{s=1}^{t-1} A^{t-1-s}Bu_s &= \eps_0^{i}\lambda^{(1-i)(k+1)} \sum_{r=1}^{i}\eps_0^{-r}\lambda^{(r-1)(k+1)}A^{(i-r)(k+1)+j}Be_{r}\\
    &= \sum_{r=1}^{i}\eps_0^{i-r}\lambda^{(r-i)(k+1)}A^{(i-r)(k+1)+j}Be_{r}\\
    &= A^jBe_i + \sum_{r=1}^{i-1}\eps_0^{i-r}\lambda^{(r-i)(k+1)}A^{(i-r)(k+1)+j}Be_{r}.
\end{align*}
The second term can be bounded as
\begin{align*}
    \|\sum_{r=1}^{i-1}\eps_0^{i-r}\lambda^{(r-i)(k+1)}A^{(i-r)(k+1)+j}Be_{r}\| &\le \sum_{r=1}^{i-1}\eps_0^{i-r}\lambda^{(r-i)(k+1)}\|A^{(i-r)(k+1)+j}B\|\\
    &\le \lambda^{j+1} \sum_{r=1}^{i-1}\eps_0^{i-r} \le (d_u-1)\lambda^{2k}\eps_0.
\end{align*}
Let $(\hat{M}_j)_i$ denote the $i$-th column of $\hat{M}_j$, then we have
\begin{align*}
    \|(\hat{M}_j)_i - (A^jB)_i\| &= \|\frac{x_i^j}{\xi_i} - (A^jB)_i\|\\
    &\le \frac{1}{\xi_i}\|A^{t-1}x_1 + \sum_{s=1}^{t-1}A^{t-1-s}w_s\| + \|\frac{1}{\xi_i}\sum_{s=1}^{t-1}A^{t-1-s}Bu_s - (A^jB)_i\|\\
    &\le \frac{1}{\xi_i}t\lambda^{t-1} + (d_u-1)\lambda^{2k}\eps_0 \\
    &\le t\eps_0^{i}\lambda^{2k} + (d_u-1)\lambda^{2k}\eps_0\le 3d_uk\lambda^{2k}\eps_0.
\end{align*}
Thus we can bound the Frobenius norm of $\hat{M}_j - A^jB$ by 
\begin{align*}
    \|\hat{M}_j - A^jB\|_F^2  = \sum_{i=1}^{d_u} \|(\hat{M}_j)_i - (A^jB)_i\|^2 \le 9d_u^3k^2\lambda^{4k}\eps_0^2.
\end{align*}
\end{proof}
We show that by our choice of $\eps_0$, $\|C_0 - C_k\|_F$, $\|C_1 - Y\|_F$ are sufficiently small to guarantee $\hat{A}$ and $A$ are close. 
\begin{restatable}{lemma}{frobA}
\label{lem:frobA}
Algorithm \ref{alg:sysid_bias} outputs $\hat{A}$ such that
$
\|\hat{A} - A\| \le \eps.
$
\end{restatable}
\begin{proof}
By Lemma \ref{lem:frob}, for all $j$, $\|\hat{M}_j - A^jB\|_F\le 3d_u^2k\lambda^{2k}\eps_0$. Let $C_k = [B\ AB\ A^2B\ \cdots A^{k-1}B]$, and $Y = [AB\ A^2B\ \ \cdots A^kB]$. We have 
\begin{align*}
    \|C_0 - C_k\|_F^2 = \sum_{j=0}^{k-1} \|\hat{M}_j - A^jB\|_F^2 \le 9d_u^4k^3\lambda^{4k}\eps_0^2.
\end{align*}
Similarly, $\|C_1 - Y\|^2_F\le 9d_u^4k^3\lambda^{4k}\eps_0^2$.

Recall that $A$ is the unique solution to the system of equations in $X$: $XC_k = Y$. Let $A_i$ denote the $i$-th row of $A$, and let $\hat{A}_i$ denote the $i$-th row of $\hat{A}$. By Lemma 22 in \citep{hazan2019nonstochastic}, as long as $\|C_0 - C_k\|_F\le \sigma_{min}(C_k)$, 
$$
\|A_i - \hat{A}_i\| \le \frac{\|C_1 - Y\|_F + \|C_0 - C_k\|_F\|A_i\|}{\sigma_{min}(C_k) - \|C_0 - C_k\|_F}
$$
By our assumption, $\|(C_kC_k^\top)^{-1}\|\le \kappa$, so $\sigma_{min}(C_k) \ge \kappa^{-1/2}$. We have 
$$
\|C_0 - C_k\|_F\le  3d_u^2k^2\lambda^{2k}\eps_0 \le \frac{\eps}{2\lambda^k d_x\sqrt{\kappa}}\le \kappa^{-1/2}/2\le \sigma_{min}(C_k).
$$
Further notice that $\|A\|\le \lambda$ implies $\|A_i\|\le \|A\|_F\le \lambda\sqrt{d_x}$, 
\begin{align*}
    \|A_i - \hat{A}_i\| &\le \frac{3d_u^2k^2\lambda^{2k}\eps_0(1 + \lambda\sqrt{d_x})}{\kappa^{-1/2}-3d_u^2k^2\lambda^{2k}\eps_0} \le \frac{\eps}{\sqrt{d_x}}.
\end{align*}
Finally, we have 
$$
\|A - \hat{A}\| \le \|A - \hat{A}\|_F = \sqrt{\sum_{i=1}^{d_x} \|A_i - \hat{A_i}\|^2}\le\eps.
$$
\end{proof}

\begin{lemma}
\label{lem:cost_1}
The total cost of estimating A, B starting from $\|x_1\|\le 1$ is bounded by 


$$G(10^5\lambda^{10k}\eps^{-2}\kappa d_x^{2}k^{5}d_u^{5})^{d_u}.$$
\end{lemma}
\begin{proof}
The magnitude of the state and control is bounded by
\begin{align*}
    \|x_t\|^2 + \|u_t\|^2\le 2\lambda^{2t-2}\eps_0^{-2d_u} \le 2\lambda^{4d_uk}\eps_0^{-2d_u}&= 2(\lambda^{4k}\eps_0^{-2})^{d_u} = 2(10^4\eps^{-2}d_u^4k^4\lambda^{10k}d_x^2\kappa)^{d_u}
\end{align*}
By Assumption \ref{a:lipschitz}, taking $D^2 = 2(10^4\eps^{-2}d_u^4k^4\lambda^{10k}d_x^2\kappa)^{d_u}$,  
$$
c_t(x_t, u_t) \le \|\nabla_{(x, u)} c_t(x_t, u_t)\|\|(x_t, u_t)\| \le 2GD^2.
$$
Summing over $(k+1)d_u \le 2kd_u$ iterations, the total cost is upper bounded by
$$
8Gkd_u(10^4\eps^{-2}d_u^4k^4\lambda^{10k}d_x^2\kappa)^{d_u} \le G(10^5\lambda^{10k}\eps^{-2}\kappa d_x^{2}k^{5}d_u^{5})^{d_u}.
$$
\end{proof}

Using our choice of $\eps$ and $\lambda$, the total cost is bounded by
\begin{align*}
    G(10^5\lambda^{10k}\eps^{-2}\kappa d_x^{2}k^{5}d_u^{5})^{d_u} &\le G(10^{25k}\beta^{10k}\gamma'^{-4}\kappa d_x^{6}k^{5}d_u^{5}\kappa'^{16})^{d_u}\\
    &\le G(10^{30k}\beta^{10k}\kappa d_x^{6}k^{5}d_u^{5}\kappa'^{24})^{d_u}\\
    &\le G(10^{30k}\beta^{10k}\kappa d_x^{18}k^{5}d_u^{5}C^{12})^{d_u}\\
    &\le G(10^{40k}\beta^{82k} d_x^{18}k^{30}d_u^{5}\kappa^{25})^{d_u}\\
    &= 2^{O(\mL\log\mL)}.
\end{align*}

\section{Proofs for Section \ref{sec:phase2}}
In this section we prove that a $(\tilde{\kappa}, \tilde{\gamma})$ strongly stable controller can be obtained by solving the SDP in Algorithm \ref{alg:k_recovery}. We first argue that for two systems close in spectral norm, a strongly stable controller for one system is also strongly stable for ther other system.
\begin{restatable}{lemma}{close}
\label{lem:close_lemma}
If $K$ is $(\kappa, \gamma)$ strongly stable for a system $(A, B)$ with $\kappa \ge 1$, and if $\hat{A}$, $\hat{B}$ satisfy $\|A - \hat{A}\| \le \eps$, $\|B - \hat{B}\| \le \eps$, then $K$ is $(\kappa, \gamma - 2\eps\kappa^2)$ strongly stable for $(\hat{A}, \hat{B})$.
\end{restatable}
\begin{proof}
    By definition, we have
\begin{align*}
    \hat{A} + \hat{B}K &= A + BK -A -BK +  \hat{A} + \hat{B}K\\
    &= HLH^{-1} + (\hat{A} - A) + (\hat{B} - B)K\\
    &= H(L + H^{-1}(\hat{A} - A +(\hat{B} - B)K)H)H^{-1}
\end{align*}
The lemma follows by observing that 
$$\|L + H^{-1}(\hat{A} - A +(\hat{B} - B)K)H\| \le 1 - \gamma + \kappa\eps(1 + \kappa) \le 1-\gamma+2\eps\kappa^2.
$$
\end{proof}

Now, we use Lemma \ref{lem:close_lemma} twice to show that the recovered controller $\hat{K}$ is strongly stable for the original system $(A, B)$. The following lemma computes $\tilde{\kappa}, \tilde{\gamma}$ in terms of $\eps$.
\begin{restatable}{lemma}{est_k}
Algorithm \ref{alg:k_recovery} returns $\hat{K}$ that is $(\tilde{\kappa}, \tilde{\gamma} )$ strongly stable for $A$ and $B$, where 
$$\tilde{\kappa} = \big(\frac{\kappa'^42d_x}{\gamma' - 2\eps\kappa'^2}\big)^{1/2},\ \  \tilde{\gamma} = \frac{\gamma' - 2\eps\kappa'^2}{4d_x\kappa'^4}- 2\eps\tilde{\kappa}^2.$$
\end{restatable}
\begin{proof}
We show in Section \ref{sec:sdp} that a $(\kappa', \gamma')$ strongly stable controller exists for $(A, B)$. Let $K$ be a $(\kappa', \gamma')$ strongly stable controller. By Lemma \ref{lem:frob}, \ref{lem:frobA}, and \ref{lem:close_lemma}, $K$ is $(\bar{\kappa}, \bar{\gamma})$ strongly stable for $\hat{A}, \hat{B}$, where $\bar{\kappa} = \kappa', \bar{\gamma}= \gamma' - 2\eps\kappa'^2$. With knowledge of $\bar{\kappa}, \bar{\gamma}$, we can set the trace upper bound appropriately to extract a strongly stable controller from a feasible solution of the SDP. Specifically, we set 
$$\nu = \frac{2\bar{\kappa}^4d_x}{\bar{\gamma}}$$ 
as in Lemma \ref{lem:sdp}, and the SDP is feasible. We obtain $\hat{K}$ that is $(\hat{\kappa}, \hat{\gamma})$ strongly stable for the system $\hat{A}, \hat{B}$, where $\hat{\kappa} = \frac{\bar{\kappa}^2\sqrt{2d_x}}{\sqrt{\bar{\gamma}}} = \big(\frac{\kappa'^42d_x}{\gamma' - 2
\eps\kappa'^2}\big)^{1/2}$, $\hat{\gamma} = \frac{\bar{\gamma}}{4d_x\bar{\kappa}^4} = \frac{\gamma' - 2\eps\kappa'^2}{4d_x\kappa'^4}$.
We apply Lemma \ref{lem:close_lemma} again and conclude that $\hat{K}$ is $(\hat{\kappa}, \hat{\gamma} - 2\eps\hat{\kappa}^2)$ strongly stable for $A, B.$ 
\end{proof}

With our choice of $\eps$, we compute the final values of $\tilde{\kappa}, \tilde{\gamma}$.
\begin{restatable}{lemma}{sdpesp}
Setting $\eps = \frac{\gamma'^2}{10^5d_x^2\kappa'^8}$, $\hat{K}$ returned by Algorithm \ref{alg:k_recovery} is
$ (\frac{2\kappa'^2d_x^{1/2}}{\gamma'^{1/2}}, \frac{\gamma'}{16d_x\kappa'^4})
$ strongly stable for $(A, B)$.
\end{restatable}
\begin{proof}
    With this choice of $\eps$, we have $2\eps\kappa'^2 = \frac{2\gamma'^2}{10^5d_x^2\kappa'^6}\le \frac{\gamma'}{2}$. It follows that
    $$
        \tilde{\kappa} = \big(\frac{\kappa'^42d_x}{\gamma' - 2\eps\kappa'^2}\big)^{1/2} \le \frac{2\kappa'^2 \sqrt{d_x}}{\sqrt{\gamma'}}.
  $$ 
  Therefore we have $
  2\eps\tilde{\kappa}^2 \le \frac{\gamma'}{10^2d_x\kappa'^4} 
  $. We obtain a lower bound on $\tilde{\gamma}$ as follows
  \begin{align*}
      \tilde{\gamma} = \frac{\gamma' - 2\eps\kappa'^2}{4d_x\kappa'^4}- 2\eps\tilde{\kappa}^2 &\ge \frac{\gamma' - 2\eps\kappa'^2}{4d_x\kappa'^4} - \frac{\gamma'}{10^2d_x\kappa'^4} \ge \frac{\gamma'}{8d_x\kappa'^4} - \frac{\gamma'}{10^2d_x\kappa'^4} \ge \frac{\gamma'}{16d_x\kappa'^4}.
  \end{align*}
\end{proof}

The following lemma details how we set the trace upper bound $\nu$ in the SDP, and our application of results from \cite{cohen2018online} to extract $\hat{K}$.

\begin{restatable}{lemma}{sdp}
\label{lem:sdp}
For any system $A, B$ with a $(\kappa, \gamma)$ strongly stable controller, the SDP in Algorithm \ref{alg:k_recovery} defined by $(A, B)$ with trace constraint $\nu = \frac{2 \kappa^4d_x}{\gamma}$ is feasible. Moreover, a policy $K$ such that $K$ is $(\frac{\kappa^2\sqrt{2d_x}}{\sqrt{\gamma}}, \frac{\gamma}{4d_x\kappa^4})$ strongly stable for $A, B$ can be extracted from any feasible solution of the SDP. 
\end{restatable}
\begin{proof}
We first show that the SDP is feasible. Let $K$ be the $(\kappa, \gamma)$ strongly stable controller for $(A, B)$, and consider the system with Gaussian noise $x_{t+1} = Ax_t + Bu_t + w_t$, $w_t\sim N(0, I)$. This system will converge to a steady state where the state covariance $X = \E[xx^\top]$ satisfies
$$
X = (A+BK)X(A+BK)^\top+I.
$$

Let $KXK^\top$ be the steady-state covariance of $u$ when following $K$. By Lemma 3.3 in \citep{cohen2018online}, $\Tr(X)\le \frac{\kappa^2d_x}{\gamma}$, $\Tr(KXK^\top) \le \frac{\kappa^4d_x}{\gamma}$. 

Consider the matrix 
$$
\Sigma = \begin{pmatrix}
X & XK^{\top}\\
KX & KXK^\top
\end{pmatrix}.
$$ 
By Lemma 4.1 in \citep{cohen2018online}, $\Sigma$ is feasible for the SDP if $\nu \ge \Tr(X) + \Tr(KXK^\top)$; since $\nu = \frac{2\kappa^4 d_x}{\gamma}$, $\Sigma$ is feasible for the SDP. Now let $\hat{\Sigma}$ be any feasible solution of the SDP, and write
$$\hat{\Sigma} = \begin{pmatrix}
\hat{\Sigma}_{xx} & \hat{\Sigma}_{xu}\\
\hat{\Sigma}_{xu}^\top & \hat{\Sigma}_{uu}
\end{pmatrix}.
$$
Consider $\hat{K} = \hat{\Sigma}_{xu}^\top\hat{\Sigma}_{xx}^{-1}$, which is well-defined because $\hat{\Sigma}_{xx} \succeq I$ by the steady-state constraint. As shown in Lemma 4.3 in \citep{cohen2018online}, $\hat{K}$ is $(\sqrt{\nu}, 1/(2\nu))$ strongly stable for $A, B$. Under our choice of $\nu$,
$\hat{K}$ is 
$
(\frac{\kappa^2\sqrt{2d_x}}{\sqrt{\gamma}}, \frac{\gamma}{4d_x\kappa^4})
$
strongly stable for $A, B$.
\end{proof}

\subsection{Decaying the System}
\begin{lemma}\label{lem:decay}
Let $K$ be a $(\tilde{\kappa}, \tilde{\gamma})$ strongly stable controller for the system. and $x_1$ be any starting state. Suppose $\tilde{\kappa}\ge 1$.
After following $K$ for $T_2 =\max\{ \frac{\ln(\tilde{\gamma} \|x_1\|)}{\tilde{\gamma}}, 0\}$ iterations , the final state $x_{T_2+1}$ satisfies $\|x_{T_2+1}\|\le 2\tilde{\kappa}/\tilde{\gamma}$, and the total cost is bounded by
$$
O\big(G \tilde{\kappa}^4\|x_{1}\|^3\tilde{\gamma}^{-3}\big).
$$
\end{lemma}
\begin{proof}
    Under the controller $K$, the state evolution satisfies 
    \begin{align*}
        x_{t+1} = (A+BK)^t x_1 + \sum_{i=1}^t(A+BK)^{t-i}w_i.
    \end{align*}
By definition of strong stability,
$
    \|(A+BK)^t\| \le \|H\|\|H^{-1}\|\|L\|^t \le \tilde{\kappa}(1-\tilde{\gamma})^t.
$
It follows that
\begin{align*}
    \|x_{t+1}\| &\le \tilde{\kappa}(1-\tilde{\gamma})^t\|x_1\| + \tilde{\kappa} \sum_{i=1}^t (1-\tilde{\gamma})^{t-i} \le \tilde{\kappa}(1-\tilde{\gamma})^t\|x_1\| + \frac{\tilde{\kappa}}{\tilde{\gamma}}.
\end{align*}
Let $T_2 = \max\{\frac{\ln(\tilde{\gamma} \|x_1\|)}{\tilde{\gamma}}, 0\}$. If $\ln(\tilde{\gamma}\|x_1\|) \ge 0$, we have $T_2 \ge -\frac{\ln(\tilde{\gamma}\|x_1\|)}{\ln(1-\tilde{\gamma})}$, hence $(1-\tilde{\gamma})^{T_2} \le 1/(\tilde{\gamma}\|x_1\|$) and $\|x_{T_2+1}\|\le 2\tilde{\kappa}/\tilde{\gamma}.$ Otherwise $T_2 = 0$ and $\|x_1\| \le 1/\tilde{\gamma} < 2\tilde{\kappa}/\tilde{\gamma}$. Notice that $\|x_t\|\le \tilde{\kappa}\|x_1\| + \tilde{\kappa}/\tilde{\gamma}$, $\|u_t\|\le\tilde{\kappa}^2\|x_1\|+ \tilde{\kappa}^2/\tilde{\gamma}$ for all $t\in [T_2+1]$. Taking $D = \tilde{\kappa}^2\|x_1\|+ \tilde{\kappa}^2/\tilde{\gamma}$ and assuming $\ln(\tilde{\gamma} \|x_1\|) \ge 0$, the total cost of decaying the system is bounded by 
\begin{align*}
2(T_2+1) GD^2 &=
   2G(\frac{\ln(\tilde{\gamma} \|x_1\|)}{\tilde{\gamma}}+1)(\tilde{\kappa}^2\|x_1\|+ \tilde{\kappa}^2/\tilde{\gamma})^2\\
   &\le 4G(\frac{\ln(\tilde{\gamma} \|x_1\|)}{\tilde{\gamma}}+1) \tilde{\kappa}^4 (\|x_1\|^2 + \frac{1}{\tilde{\gamma}^2})\\
   &\le 8G(\frac{\ln( \|x_1\|)}{\tilde{\gamma}}+1) \tilde{\kappa}^4 \|x_1\|^2\tilde{\gamma}^{-2}\\
   &\le 8G(\ln( \|x_1\|)+1) \tilde{\kappa}^4 \|x_1\|^2\tilde{\gamma}^{-3}\\
   &\le 16G\tilde{\kappa}^4 \|x_1\|^3\tilde{\gamma}^{-3}.
\end{align*}
The same upper bound holds for $T_2 = 0$.
\end{proof}

\section{Proofs for Section \ref{sec:nonstochastic}}
In this section we give an upper bound on quantities related to the controllability of the stabilized system $(A+B\hat{K}, B)$, and include the main results in \cite{hazan2019nonstochastic} for completeness. The following lemma is an equivalent characterization of strong controllability.

\begin{lemma}\label{lem:controllability}
A system defined by $x_{t+1} = Ax_t + Bu_t$ is $(k,\kappa)$-strongly controllable if and only if it can drive $x_1 = 0$ to any state $x_f$ where $\|x_f\| = 1$ in $k$ steps with control cost at most $\kappa$. I.e., there exists $u_1, \ldots, u_k, x_2, \ldots, x_{k+1}$ such that $x_{k+1} = x_f$, $x_{t+1} = Ax_t + Bu_t$, and 
$$
\sum_{t=1}^k \|u_t\|^2 \leq \kappa.
$$
\end{lemma}
\begin{proof}
Consider the quadratic program: 
\begin{equation}
\begin{aligned}
\min_{(u_t)_{t=1}^k} \quad & \sum_{t=1}^k \|u_t\|^2\\
\textrm{s.t.} \quad & x_{t+1} = A x_t + B u_t\\
  &x_{k+1} = x_f,\ x_1 = 0    \\
\end{aligned}
\end{equation}
Recall $C_k = [B\ AB\ \cdots\ A^{k-1}B]$, and let $(v_1, v_2, \dots, v_k)\in \mathbb{R}^{kn}$ denote the concatenation of $k$ $n$-dimensional vectors. Then this is equivalent to 
\begin{equation}
\label{opt}
\begin{aligned}
\min_{(u_t)_{t=1}^k} \quad & \sum_{t=1}^k \|u_t\|^2\\
\textrm{s.t.} \quad & C_k (u_k, u_{k-1}, \ldots, u_1)  = x_f\\
\end{aligned}
\end{equation}
Suppose the system is $(k, \kappa)$ strongly controllable, then $C_k$ has full row-rank, and $C_kC_k^\top$ is invertible with $\|(C_kC_k^\top)^{-1}\|\le \kappa$. Therefore \eqref{opt} is feasible for all unit vectors $x_f$.
By Lemma B.6 in \cite{cohen2018online}, an optimal solution to \eqref{opt} is given by $C_k^\top(C_kC_k^\top)^{-1}x_f$, and its value is at most  
$$ \sum_{t=1}^k \|u_t\|^2 = \|C_k^\top (C_kC_k^\top)^{-1}x_f\|^2 = x_f^\top (C_kC_k^\top)^{-1}x_f\le \|(C_kC_k^\top)^{-1}\| = \kappa. $$
Now suppose for any unit vector $x_f$, there exists $u_1, \ldots, u_k, x_2, \ldots, x_{k+1}$ such that $x_{k+1} = x_f$, $x_{t+1} = Ax_t + Bu_t$, and $\sum_{t=1}^k \|u_t\|^2\le \kappa$. Then \eqref{opt} is feasible for any unit vector $x_f$, implying that $C_k$ has full row-rank and $(C_kC_k^\top)$ is invertible. Moreover, the optimal value is at most $\kappa$. Let $x_f$ be the eigenvector corresponding to the largest eigenvalue of $(C_kC_k^\top)^{-1}$. Then an optimal solution to \eqref{opt} is $C_k^\top(C_kC_k^\top)^{-1}x_f$, and the value satisfies $\|C_k^\top(C_kC_k^\top)^{-1}x_f\|^2 = x_f^\top(C_kC_k^\top)^{-1}x_f\le \kappa$. We conclude that $\|(C_kC_k^\top)^{-1}\|\le \kappa$, and the system is $(k, \kappa)$ strongly controllable.
\end{proof}

Using our characterization, we show an upper bound on the controllability parameter of $(A+BK, B)$ where $K$ is any linear controller with a bounded spectral norm.
\begin{restatable}{lemma}{controllable}
\label{lem:controllable}
Suppose $(A, B)$ is $(k, \kappa)$ strongly controllable and $\|A\|, \|B\|\le \beta$. Let $K$ be a linear controller with $\|K\|\le \kappa'$, then the system $(A+BK, B)$ is $(k, \kappa_0)$ strongly controllable, with $\kappa_0 = 4\kappa'^2k^2\beta^{2k}\kappa$.
\end{restatable}
\begin{proof}
Let $C_k = [B\ AB\ \cdots\ A^{k-1}B]$. By the definition of strong controllability, $C_k$ has full row-rank, and under the noiseless system $x_{t+1} = Ax_t + Bu_t$, any state is reachable by time $k+1$ starting from $x_1 = 0$. We will show that any state is reachable at time $t+1$ for the system $(A+BK, B)$ as well. Let $v\in\mathbb{R}^m$ be an arbitrary state, and the sequence of controls $(u_1, u_2, \ldots, u_k) = C_k^\top(C_kC_k^\top)^{-1}v$ can be used to reach $v$ from initial state $x_1 = 0$, i.e. 
$$
x_{k+1} = \sum_{i=1}^k A^{k-i}Bu_i = C_k(u_1, u_2, \ldots, u_k) = v.
$$
Let $\{x_t\}$ denote the state trajectory under controls $\{u_t\}$, where $x_{k+1} = v$. Consider the system $y_{t+1} = (A+BK)y_t + Bz_t = Ay_t + B(z_t + Ky_t)$, where $y_t$'s are states and $z_t$'s are controls. We claim that the sequence of controls $z_t = u_t - Ky_t$ can be used to drive the system to $v$ in $k+1$ steps from initial state $y_1 = 0$. Let $\{y_t\}$ denote the system's trajectory under controls $\{z_t\}$. For our base case, we have $y_2 = B(z_1+Ky_1) =Bu_1 = x_2$, since $y_1 = x_1=0, z_1 = u_1 - Ky_1$. Assume $x_t = y_t$ for some $t \le k$. For $t+1$,
$y_{t+1} = Ay_t + B(z_t + Ky_t) = Ay_t + Bu_t = Ax_t + Bu_t = x_{t+1}$. We conclude that the trajectories $\{x_t\}$ and $\{y_t\}$ are the same and $v = y_{k+1}$. Since we can write $y_{k+1} = \sum_{i=1}^k (A+BK)^{k-i}Bz_i$, $y_{k+1}$ is in the range of the matrix $C_k' = [B\ (A+BK)B\ \cdots\ (A+BK)^{k-1}B]$; therefore $C_k'$ has full row-rank.  

Now we show the controls $\{z_t\}$ satisfy $\sum_{t=1}^k \|z_t\|^2 \le 4\kappa'^2k^2\beta^{2k}\kappa\|v\|^2$. By our choice of $z_t$, we have $z_t = u_t - Ky_t = u_t - Kx_t$; therefore $\sum_{t=1}^k \|z_t\|^2 \le 2\sum_{t=1}^k (\|u_t\|^2 + \|K\|^2\|x_t\|^2)$. By our choice of $u_t$, we have
$$
\sum_{t=1}^k\|u_t\|^2 = \|C_k^\top (C_kC_k^\top)^{-1}v\|^2 = v^\top(C_kC_k^\top)^{-1}v \le \kappa \|v\|^2.
$$
Further, the trajectory $\{x_t\}_{t=1}^k$ satisfies
\begin{align*}
    \|x_t\|^2 = \|\sum_{i=1}^{t-1}A^{t-1-i}Bu_i \|^2\le k\sum_{i=1}^{t-1}\|A^{t-1-i}B\|^2\|u_i\|^2 \le k\beta^{2k}\kappa\|v\|^2
\end{align*}
Hence we have 
\begin{align*}
    \sum_{t=1}^k \|z_t\|^2 \le 2\kappa\|v\|^2 + 2\kappa'^2k^2\beta^{2k}\kappa\|v\|^2 \le  4\kappa'^2k^2\beta^{2k}\kappa\|v\|^2. 
\end{align*}
By Lemma \ref{lem:controllability}, $(A+BK, B)$ is $(k, 4\kappa'^2k^2\beta^{2k}\kappa)$ strongly controllable.
\end{proof}

Algorithm \ref{alg:gpc_md} is the main algorithm (Algorithm 1) in \cite{hazan2019nonstochastic}, where $T_0, \eta, H$ are internal parameters that can be set by the learner. In line 10, let $\Pi_\mathcal{M}$ denote projection onto the set $\mathcal{M}$, and let $f_t$ denote the surrogate cost at time $t$ as in Definition 11 of \cite{hazan2019nonstochastic}. Theorem \ref{thm: gpc} gives the regret bound for the algorithm when the internal parameters are set appropriately. 
\begin{algorithm}[]
\caption{Adversarial Control via System Identification} 
\label{alg:gpc_md}
\begin{algorithmic}[1]
\STATE \textbf{Input:} Number of iterations $T$, $\tilde{\gamma}, \hat{K}$ such that $\hat{K}$ is  $(\tilde{\kappa}, \tilde{\gamma})$ strongly stable, $\kappa^*, k$ such that $(A+B\hat{K}, B)$ is $(k, \kappa^*)$ strongly controllable, and $\kappa^*\ge\tilde{\kappa}$.
\STATE \underline{Phase 1:} \textbf{System Identification.}
\STATE Call Algorithm 2 in \citep{hazan2019nonstochastic} with a budget of $T_0$ rounds to obtain system estimates $\tilde{A}, \tilde{B}$.
\STATE \underline{Phase 2:} \textbf{Robust Control.}\\
Define the constraint set $\M = \{M = \{M^{0} \ldots M^{H-1}\}: \|M^{i-1}\| \leq \kappa^4(1-\gamma)^i \}$.
\STATE Initialize $\hat{w}_{T_0} = x_{T_0+1}$ and $\hat{w}_t = 0$ for $t < T_0$.
\FOR{$t = T_0+1, \ldots, T$}
\STATE Choose the action: 
\[u_t = \hat{K} x_t+ \sum_{i=1}^{H} M_t^{i-1} \hat{w}_{t-i}.\]
\STATE Observe the new state $x_{t+1}$ and cost $c_t(x_t, u_t)$.
\STATE Record estimate $\hat{w}_t=x_{t+1}-\tilde{A} x_t-\tilde{B}u_t$.
\STATE Update:
\[M_{t+1} = \Pi_{\M}(M_t - \eta \nabla f_{t}(M_t|\tilde{A}, \tilde{B}, \{\hat{w}\})) \]
\ENDFOR
\end{algorithmic}
\end{algorithm}
\begin{theorem}\label{thm: gpc}[Theorem 12 in \cite{hazan2019nonstochastic}]
Suppose $\hat{K}$ is $(\tilde{\kappa}, \tilde{\gamma})$ strongly stable for $(A, B)$, and the system $(A+B\hat{K}, B)$ is $(k, \kappa^*)$ strongly controllable. In addition, assume that the noise sequence $w_t$ satisfies $\|w_t\|\le W$ for all $t$. Then Algorithm \ref{alg:gpc_md} with $H = \Theta(\tilde{\gamma}^{-1}\log((\kappa^*)^2 T)), \eta = \Theta(GW\sqrt{T})^{-1}, T_0 = \Theta(T^{2/3}\log(1/\delta))$, incurs regret upper bounded by 
\begin{align*}
 \text{Regret} &= O(\poly(\kappa^*, \tilde{\gamma}^{-1}, k, d_x,d_u, G, W)T^{2/3}\log(1/\delta)).
 \end{align*}
with probability at least $1-\delta$ for controlling an unknown LDS.
\end{theorem}

\section{Proofs for Lower Bound for Randomized Black-box Control Algorithms}

\begin{lemma} \label{lem:distr_restate}
Let $T = d_x/8$. There exists a sequence of orthonormal matrices $V_1, \ldots, V_{T}$, such that for $t\in [T]$, $V_t$ only depends on $x_1, \ldots, x_t, u_1, \ldots, u_t$ and they satisfy the following condition: \\Let $r_t$ denote the rank of $\text{span}(x_1, \ldots, x_t,$ $ u_1, \ldots, u_t)$, and denote the first $r_t$ columns of $V_t$ as $V_t^{\|}$, and the last $d-r_t$ columns of $V_t$ as $V_t^\bot$.
Let 
$h_t = (V_{t-1}^{\bot})^\top x_t$, then for all $t\in[T]$, conditioned on $x_1, x_2, \ldots, x_t,$ $u_1, u_2, \ldots, u_{t}, Au_1, \ldots, Au_{t-1}$, we have $(V_{t}^\bot)^\top x_{t+1} = c_t + z_t$, where the coordinates of $z_t$ are iid normally distributed, i.e. $z_t(i) \sim N(0, \frac{\gamma \|h_t\|^2}{d})$.
\end{lemma}
\begin{proof}
Fix $t\le T$, and condition on $x_1, \ldots, x_t,$ $u_1, \ldots, u_{t-1}, Au_1, \ldots, Au_{t}$. Let $V_1, \ldots, V_{t-1}$ be constructed as in Corollary \ref{cor:lower_bound_construction}. 
By construction, the first $r_{t-1}$ columns of $V_{t-1}$, denoted as $V_{t-1}^{\|}$, form a basis for $\text{span}(x_1, \ldots, x_{t-1}, u_1, \ldots, u_{t-1})$. Recall the last $d-r_{t-1}$ columns of $V_{t-1}$ is denoted as $V_{t-1}^\bot$. We have
\begin{align*}
    (V_t^\bot)^\top x_{t+1} &= (V_t^\bot)^\top Ax_t + (V_t^\bot)^\top u_t\\
    &= (V_t^\bot)^\top AV_{t-1} V_{t-1}^\top x_t \tag{$(V_t^\bot)^\top u_t = 0$}\\
    &= (V_t^\bot)^\top \left[
\begin{array}{c|c}
AV_{t-1}^{\|} & AV_{t-1}^\bot 
\end{array}
\right] V_{t-1}^\top x_t\\
&= \left[
\begin{array}{c|c}
(V_t^\bot)^\top AV_{t-1}^{\|} & (V_t^\bot)^\top AV_{t-1}^\bot 
\end{array}
\right] V_{t-1}^\top x_t\\
&= \left[
\begin{array}{c|c}
(V_t^\bot)^\top AV_{t-1}^{\|} & (V_t^\bot)^\top AV_{t-1}^\bot 
\end{array}
\right] \left[
\begin{array}{c}
(V_{t-1}^{\|})^\top \\
(V_{t-1}^{\bot})^\top
\end{array}
\right] x_t\\
&= (V_t^\bot)^\top AV_{t-1}^{\|}(V_{t-1}^{\|})^\top x_t + (V_t^\bot)^\top AV_{t-1}^\bot (V_{t-1}^{\bot})^\top x_t 
\end{align*}
Denote $(V_t^\bot)^\top AV_{t-1}^{\|}(V_{t-1}^{\|})^\top x_t$ as $c_t \in \mathbb{R}^{d-r_t}$, and recall $h_t = (V_{t-1}^{\bot})^\top x_t$. We have
\begin{align*}
    (V_t^\bot)^\top x_{t+1} = c_t + G_th_t,
\end{align*}
where $G_t = (V_t^\bot)^\top AV_{t-1}^\bot$. By Corollary \ref{cor:lower_bound_construction}, conditioned on $x_1, \ldots, x_t, u_1, \ldots, u_t$, $Au_1, \ldots, Au_{t-1}$, $G_t\sim N(d-r_t, d-r_{t-1}, \frac{\gamma}{d_x})$, and therefore the coordinates of $z_t = G_th_t$ are iid normally distributed with zero mean and $\frac{\gamma \|h_t\|^2}{d_x}$ variance.
\end{proof}

\begin{lemma}\label{lem:lower_bound_magnitude_restate}
Let $V_1, \ldots, V_T$ be as in Lemma \ref{lem:distr}, and $T \le d_x/8$. Let $h_t = (V_{t-1}^{\bot})^\top x_t$, with probability at least $1-\exp(-\frac{d_x}{25})$, conditioned on $x_1, x_2, \ldots, x_t, u_1, u_2, \ldots, u_t$, $Au_1, \ldots, Au_{t-1}$, we have $\|(V_{t}^\bot)^\top x_{t+1}\|^2 \ge \frac{\gamma \|h_t\|^2}{20}$.
\end{lemma}
\begin{proof}
By Lemma \ref{lem:distr}, conditioned on $x_1, \ldots, x_t$, $u_1, \ldots, u_t$, $Au_1, \ldots, Au_t$, we have $h_{t+1} = (V_{t}^{\bot})^\top x_{t+1} \sim N(c_t, \frac{\gamma \|h_t\|^2}{d_x} I)$. There exists a rotation $R$ of $h_{t+1}$, such that $Rh_{t+1} \sim N(\|c_t\|e_1, \frac{\gamma \|h_t\|^2}{d_x} I)$.
Let $Rh_{t+1}(i)$ denote the $i$-th coordinate of the vector $Rh_{t+1}$, then we have $\sum_{i=2}^{d_x-r_t}Rh_{t+1}(i)^2 \sim \frac{\gamma\|h_t\|^2}{d_x} \chi_{d_x-r_t-1}$ follows a chi-square distribution. By Lemma 1 in \citet{laurent2000}, for a random variable $Y \sim \chi_k$, $\Pr[Y\le k - 2\sqrt{kx}]\le \exp(-x)$. Therefore for $t \le T < \frac{d_x}{8}$, $r_t\le 2t\le \frac{d_x}{4}$, we have
\begin{align*}
    \Pr[\sum_{i=2}^{d_x-r_t}Rh_{t+1}(i)^2 \le \frac{\gamma \|h_t\|^2}{20}] &= \Pr[\frac{d_x}{\gamma\|h_t\|^2} \sum_{i=2}^{d_x-r_t}Rh_{t+1}(i)^2 \le \frac{d_x}{20}]\\
    &\le  \Pr[\frac{d_x}{\gamma\|h_t\|^2} \sum_{i=2}^{d_x-r_t}Rh_{t+1}(i)^2 \le d_x-r_t-1 - 2\sqrt{(d_x-r_t-1)d_x/25}] \\
    &\le \exp(-\frac{d_x}{25})
\end{align*}
We conclude that
\begin{align*}
    \Pr[\|(V_t^\bot)^\top x_{t+1}\|^2 \ge \frac{\gamma \|h_t\|^2}{20}] = \Pr[\|Rh_{t+1}\|^2 \ge \frac{\gamma \|h_t\|^2}{20}] &\ge \Pr[ \sum_{i=2}^{d_x-r_t} Rh_{t+1}(i)^2 \ge \frac{\gamma \|h_t\|^2}{20}] \\
    &\ge 1-\exp(-\frac{d_x}{25}).
\end{align*}
\end{proof}

\begin{lemma} \label{lem:lower_bound_construction}
Consider the observation model, where $A \sim N(d, d, \gamma)$, and a player can make queries defined by vectors $q_1, q_2, \ldots, q_T$, $T\le d$. In turn, the player observes $w_1 = Aq_1, w_2 = Aq_2, \ldots, w_T = Aq_T$. For any $t \le T$, the player is allowed to choose $q_t$ as a deterministic function of the previous queries and observations. Let $r_t$ denote the rank of $\text{span}(q_1, \ldots, q_t)$. Then for all $t\le T$, there exists an orthonormal matrix $V_t$ that can be constructed only as a function of $q_1, \ldots, q_t$, such that with $V_t^\bot$ denoting the last $d-r_t$ columns of $V_t$, the following hold:
\begin{enumerate}
    \item Conditioned on $q_1, q_2, \ldots, q_t, w_1, \ldots, w_{t-1}$, $(V_t^{\bot})^{\top} A V_{t-1}^\bot \sim N(d-r_t, d-r_{t-1}, \gamma)$.
    \item Conditioned on $q_1, q_2, \ldots, q_t, w_1, \ldots, w_{t}$, $A V_{t}^\bot \sim N(d, d-r_t, \gamma)$.
\end{enumerate}
\end{lemma}
\begin{proof}
We first construct $V_1, \ldots, V_T$. For $t\le T$, let $r_t$ be the rank of $\text{span}(q_1, q_2, \ldots, q_{t})$, $r_t\le t$, and denote the normalized component of $q_t$ that lies outside of $\text{span}(q_1, q_2, \ldots, q_{t-1})$ as $\tilde{q}_t$. Let $W_1, \ldots, W_T$ be orthonormal matrices, such that if $\tilde{q}_t$ = 0, $W_t = I$; otherwise, the first $r_{t-1}$ columns of $W_t$ are standard basis vectors $e_1, \ldots, e_{r_{t-1}}$, and the $r_t$-th column of $W_t$, denoted as $z_t$, is such that $W_1 W_2\cdots W_{t-1} z_t = \tilde{q_t}$. Such a $W_t$ exists because the product $W_1\cdots W_{t-1}$ is an orthonormal matrix and thus is full rank. Moreover, $z_t$ is orthogonal to $e_1, \ldots e_{r_t-1}$ by the construction of $\tilde{q}_t$. Let $V_t = W_1W_2\cdots W_t$. Then by definition, $V_t$ is orthonormal, and the first $r_t$ columns of $V_t$ form a basis for $\text{span}(q_1, q_2, \ldots, q_t)$. Moreover, $V_t$ only depends on $q_1, \ldots, q_t$. Denote the first $r_t$ columns of $V_t$ by $V_t^{\|}$, and recall that the last $d-r_t$ columns of $V_t$ is denoted by $V_t^\bot$. Now we prove the lemma by induction. \\
\textbf{Base case. } Define $V_0 = 0$. Since $q_1$ is chosen without any observations, it is independent of $A$, and hence $V_1$ is independent of $A$. Therefore conditioned on $q_1$, $AV_1^{\|}$ is independent of $AV_1^{\bot}$, and $(V_1^{\bot})^{\top} A V_{0}^\bot  = (V_1^{\bot})^{\top} A I \sim N(d-r_1, d, \gamma)$, $AV_1^\bot \sim N(d, d-r_1, \gamma)$. Since $w_1$ only depends on $AV_1^{\|}$, it is independent of $AV_1^{\bot}$. We conclude that conditioned on $q_1$ and $w_1$, $AV_1^\bot \sim N(d, d-r_1, \gamma)$. \\
\textbf{Inductive step. } Suppose for all $s < t$, the two conditions in the lemma hold. By definition, $q_t$ is a deterministic function of $q_1, \ldots, q_{t-1}, w_1, \ldots, w_{t-1}$, so by the inductive hypothesis, conditioned on $q_1, \ldots, q_{t-1}, q_t, w_1, \ldots, w_{t-1}$, $A V_{t-1}^\bot \sim N(d, d-r_{t-1}, \gamma)$, and we can obtain $W_t$ and $V_t$. Since $V_t$ is only a function of $q_1, \ldots, q_t$, we have $(V_t^\bot)^\top A V_{t-1}^\bot \sim N(d-r_t, d-r_{t-1}, \gamma)$. Denote the last $d-r_t$ columns of $W_t$ as $Z_t$. Now observe 
\begin{align*}
   V_{t}^{\bot} &= V_{t-1} W_t \left[
\begin{array}{c}
0_{r_t\times r_t} \\ I_{d-r_t}
\end{array}
\right] =  \left[
\begin{array}{c|c}
V_{t-1}^{\|} &
V_{t-1}^{\bot}
\end{array}
\right] Z_t 
\end{align*}
By construction the columns of $Z_t$ are orthogonal to $e_1, \ldots, e_{r_{t-1}}$, therefore their first $r_{t-1}$ coordinates are all zero, and we can write $Z_t = \left[
\begin{array}{c}
0_{(r_{t-1}) \times (d-r_t)} \\ \tilde{Z}_t
\end{array}
\right]$, where $\tilde{Z}_t \in \mathbb{R}^{(d-r_{t-1}) \times (d-r_t)}$ have orthonormal columns. Therefore we have
\begin{align*}
  V_t^\bot &= 
V_{t-1}^{\bot}
\tilde{Z}_t
\end{align*}
Since $\tilde{Z}_t$ is independent of $A V_{t-1}^\bot$, we have $A V_{t}^\bot = AV_{t-1}^\bot \tilde{Z}_t \sim N(d, d-r_t, \gamma)$. Now we need to show that this distribution doesn't change conditioned on $w_t$. If $q_t \in \text{span}(q_1, \ldots, q_{t-1})$, then $w_t = Aq_t$ can be determined by $w_1, \ldots, w_{t-1}$, so the distribution of $A V_{t}^\bot$ remains the same conditioned on $w_t$. Now assume $q_t \notin \text{span}(q_1, \ldots, q_{t-1})$, and $r_t = r_{t-1}+1$. Since $w_t$ is determined by $AV_t^{\|}$, it suffices to show that $AV_t^{\|}$ is independent of $AV_t^{\bot}$ conditioned on $q_1, \ldots, q_t, w_1, \ldots, w_{t-1}$. Consider the following decomposition 
$$
A V_t^{\|} = A \left[
\begin{array}{c|c}
V_{t-1}^{\|} & V_te_{r_t}
\end{array}
\right] = \left[
\begin{array}{c|c}
AV_{t-1}^{\|} & AV_te_{r_t}
\end{array}
\right] .
$$
By the construction of $V_{t-1}^{\|}$, $AV_{t-1}^{\|}$ can be determined by $w_1, \ldots, w_{t-1}$. Therefore, by the inductive hypothesis, $AV_{t-1}^{\|}$ is independent of $AV_{t-1}^\bot\tilde{Z}_t = AV_{t}^\bot$. Now we expand $V_te_{r_t}= V_{t-1}W_te_{r_t}$, and as before, let $z_t = W_te_{r_t}$. Since $z_t$ is orthogonal to $e_1, \ldots, e_{r_{t-1}}$, the first $r_{t-1}$ coordinates of $z_t$ are zero. Let the last $d-r_{t-1}$ coordinates of $z_t$ be $y_t$, then we have $V_te_{r_t} = V_{t-1}z_t = V_{t-1}^{\bot}y_t$, and $y_t$ is orthogonal to the columns of $\tilde{Z}_t$. By the inductive hypothesis, $AV_{t-1}^\bot \sim N(d, d-r_{t-1}, \gamma)$, so $AV_{t-1}^\bot y_t$ is independent of $AV_{t-1}^\bot \tilde{Z}_t$. We conclude that $AV_te_{r_t}$ is independent of $AV_t^\bot$, so $AV_t^{\|}$ is independent of $AV_t^\bot$ and conditioned on $w_t$, $AV_t^\bot\sim N(d, d-r_t, \gamma)$.
\end{proof}

\begin{corollary} \label{cor:lower_bound_construction}
Consider an alternative observation model, where $A\sim N(d, d, \gamma)$, and a player can make two queries at a time: $p_1, q_1, p_2, q_2, \ldots, p_T, q_T, T<d/2$. The player observes $v_t = Ap_t, w_t = Aq_t$ for $t\in [T]$, and the player can choose $p_t, q_t$ as deterministic functions of $\{p_s\}_{s=1}^{t-1},$ $\{q_s\}_{s=1}^{t-1},$ $\{v_s\}_{s=1}^{t-1},$ $\{w_s\}_{s=1}^{t-1}$. Let $r_t$ denote the rank of $\text{span}(\{p_s\}_{s=1}^{t}, \{q_s\}_{s=1}^{t})$. Then for all $t\le T$, there exists an orthonormal matrix $V_t$ that can be constructed only as a function of $\{p_s\}_{s=1}^{t}, \{q_s\}_{s=1}^{t}$, such that with $V_t^\bot$ denoting the last $d-r_t$ columns of $V_t$, the following hold:
\begin{enumerate}
    \item Conditioned on $\{p_s\}_{s=1}^{t}, \{q_s\}_{s=1}^{t}, \{v_s\}_{s=1}^{t-1}, \{w_s\}_{s=1}^{t-1}$, $(V_t^{\bot})^{\top} A V_{t-1}^\bot \sim N(d-r_t, d-r_{t-1}, \gamma)$.
    \item Conditioned on $\{p_s\}_{s=1}^{t}, \{q_s\}_{s=1}^{t}, \{v_s\}_{s=1}^{t}, \{w_s\}_{s=1}^{t}$, $A V_{t}^\bot \sim N(d, d-r_t, \gamma)$.
\end{enumerate}
\end{corollary}
\begin{proof}
The proof is very similar to the proof of Lemma \ref{lem:lower_bound_construction}. 
\end{proof}

\section{Lower Bound for Deterministic Black-box Control Algorithms}\label{a:lower_bound_det}
\begin{theorem} \label{thm:lowerbound}
Let $\A$ be a deterministic black-box control algorithm. Then there exists a stabilizable system that is also $(1, 1)$-strongly controllable, and a sequence of oblivious perturbations and costs, such that with $x_1 = e_1$, and $T = d_x$, we have 
$$ \regret_T(\A)= 2^{\Omega(\mL)} . $$
\end{theorem}
Let $c_t(x, u) = \|x\|^2 + \|u\|^2$ for all $t$. Consider the noiseless system $x_{t+1} = Q^\top Vx_t + u_t$ for some $Q$ and orthogonal $V$. Under this system $w_t = 0$ for all $t$, and a stabilizing controller is $K = -Q^\top V$. Observe that $J(K)$ is constant. The system is also $(1, 1)$ strongly controllable because $B = I$. Let $V_i, Q_i \in \mathbb{R}^{1\times d_x}$ denote the rows of $V$ and $Q$, respectively. Fix a deterministic algorithm $\A$, and let $u_t = \A(x_1, x_2, \ldots, x_t, c_1, \ldots, c_t)$ be the control produced by $\A$ at time $t$. There exists $Q, V$ such that under this system, $\A$ outputs controls such that $\|x_{d_x}\|\ge 2^{d_x-1}$. 
\allowdisplaybreaks
\paragraph{The construction.}
Set $x_1 = e_1$. We construct Q and V as follows: let $y_0 = e_1$, set $V_1 = y_0^\top = e_1^\top$; for $i=1, \ldots, d_x-1$, define
$$
z_i = \begin{cases} u_i &\mbox{if }u_i\notin \text{span}(V_1^\top, \ldots, V_i^\top)\\
v\text{ s.t. }v\in \text{span}(V_1^\top, \ldots, V_i^\top)^\perp, \|v\| = 1 &\mbox{otherwise}
\end{cases}
$$
Let $y_i$ be the component of $z_i$ that is independent of $V_1^\top, \ldots, V_i^\top$,
$$
y_i = \frac{z_i - \sum_{j=1}^{i}\Pi_{V_j^\top}(z_i)V_j^\top}{\|z_i - \sum_{j=1}^{i}\Pi_{V_j^\top}(z_i)V_j^\top\|},
$$
where $\Pi_v(z)$ denotes the projection of $z$ onto vector $v$. Set $Q_i = d_iy_i^\top$ for some $d_i \neq 0$ to be specified later, and set $V_{i+1} = y_i^\top$. 

The next lemma justifies this iterative construction of $V$ by showing that the trajectory $x_1, \ldots, x_t$ is not affected by the choice of $V_i, Q_i$ for $i\ge t$. As a result, without loss of generality we can set $V_t$ after obtaining $x_t$, and set $Q_t$ after receiving $u_t$.

\begin{restatable}{lemma}{lowerboundconstruction}
\label{lowerbound_construction}
As long as $V$ is orthogonal, the states satisfy
$x_t = \sum_{i=1}^{t-1} c^t_i V_i^\top + c^t_ty_{t-1}$ for some constants $c_i^t$ that only depend on $\A$ and $\{Q_i\}_{i=1}^{t-1}$.
\end{restatable}

\begin{proof}
We prove the lemma by induction. For our base case, $x_1$ is trivially $c_1^1 e_1$ and it is fixed for all choices of $Q, V$. Set $V_1 = e_1^\top$.  Assume the lemma is true for $x_t$, and we have specified $V_i$ for $i\le t$, $Q_i$ for $i < t$. The specified rows of $V$ are orthonormal by construction. Note that by our construction, $x_t$ is obtained first, and then we set $V_t = y_{t-1}^\top$. Since $u_t$ only depends on the current trajectory up to $x_t$, it is well-defined, and we can obtain $z_t$. By definition of $y_t$, we can write $u_t = \sum_{i=1}^{t} a_i^tV_i^\top + a_{t+1}^t y_t$ for some coefficients $a_i^t$. Set $Q_t = d_ty_t^\top$ as in the lemma. The next state is then
\begin{equation}
\begin{aligned}[b]
    x_{t+1} = Q^\top Vx_t + u_t &= Q^\top V\sum_{i=1}^t c_i^t V_i^\top + \sum_{i=1}^{t} a_i^t V_i^\top + a_{t+1}^ty_t &\mbox{$V_t = y_{t-1}^\top$}\\
    &=\sum_{i=1}^t c_i^t Q^\top e_i + \sum_{i=1}^{t} a_i^t V_i^\top + a_{t+1}^ty_t &\mbox{$V$ is orthogonal}\\
    &= \sum_{i=1}^t c_i^t Q_i^\top + \sum_{i=1}^{t} a_i^t V_i^\top + a_{t+1}^ty_t\\
    &= \sum_{i=1}^{t-1} c_i^t d_iV_{i+1}^\top + c_t^td_ty_t +\sum_{i=1}^{t} a_i^t V_i^\top + a_{t+1}^ty_t &\mbox{$Q_i = d_iy_i^\top = d_iV_{i+1}$}\\
    &= \sum_{i=1}^{t} c_i^{t+1} V_i^\top + c_t^td_ty_t+ a_{t+1}^ty_t 
\end{aligned}
\label{eqn1}
\end{equation}
We have shown in the inductive step that $x_{t+1}$ does not depend on the choice of $V_{t+1}$ as long as $V$ is orthogonal, hence we can set $V_{t+1} = y_t^\top$. Moreover, $x_{t+1}$ is not affected by $Q_i$ for $i \ge t+1$ by inspection. 
\end{proof}

\paragraph{The magnitude of the state.} In this section we specify the constants $d_i$ in the construction to ensure that the state has an exponentially increasing magnitude. Let $u_i = \sum_{j=1}^{i} a_j^iV_j^\top + a_{i+1}^i y_i$, $x_i = \sum_{j=1}^{i-1} c^i_j V_j^\top + c_i^i y_{i-1}$. Set $d_i = \sign(c_i^i)\sign(a_{i+1}^i)\cdot 2$. The quantities $c_i^i$ and $a_{i+1}^i$ are well-defined when we set $Q_i$ after obtaining $x_i$ and $u_i$. Intuitively, $Q^\top V$ applied to $x_i$ aligns $y_{i-1}$ to $y_i$, and we grow the magnitude of the $y_i$ component in $x_{t+1}$ multiplicatively.

\begin{restatable}{lemma}{lowerboundmagnitude}
The states satisfy $x_t = \sum_{i=1}^t c_i^t V_i^\top$, and $|c_t^t|\ge 2|c_{t-1}^{t-1}|$.
\end{restatable}

\begin{proof}
By equation \ref{eqn1} in Lemma \ref{lowerbound_construction}, we can express
$x_{t+1} = \sum_{i=1}^{t} c_i^{t+1} V_i^\top + c_t^td_ty_t+ a_{t+1}^ty_t $. As we claimed before, since $x_{t+1}$ does not depend on the choice of $V_{t+1}$, we set $V_{t+1} = y_t$, and write $x_{t+1} = \sum_{i=1}^{t+1} c_i^{t+1} V_i^\top$.
By our choice of $d_t$, we have
\begin{align*}
    c_{t+1}^{t+1} = c_t^td_t + a_{t+1}^t = \sign(c_t^t)\sign(a_{t+1}^t) \cdot 2c_t^t + a_{t+1}^t = \sign(a_{t+1}^t)(2|c_t^t| + |a_{t+1}^t|).
\end{align*}
We conclude that $|c_{t+1}^{t+1}| = 2|c_t^t| + |a_{t+1}^t| \ge 2|c_t^t|.$
\end{proof}

Observe that $x_1 = c_1^1e_1$ where $|c_1^1| = 1$; therefore we have $\|x_{d_x}\| \ge |c_{d_x}^{d_x}| \ge 2^{d_x-1}$.

\paragraph{Size of the system.} Our construction only requires $Q_1, \ldots, Q_{d_x-1}$ to be specified, and without loss of generality we take $Q_{d_x} = d_{d_x}V_1 = 2V_1$. By inspection, $Q$ can be written as $Q = DPV$, where $D = Diag(d_1, d_2, \ldots, d_{d_x})$ and $P$ is a permutation matrix that satisfies $(PV)_i = V_{i+1\pmod{d_x}}$. Therefore the spectral norm of $Q^\top V$ is at most $\|Q\|\|V\| \le 2$. We conclude that for this system, $\mL = d_u+d_x+7$, and the total cost is at least $2^{\Omega(\mL)}$.

\end{document}